\documentclass{ecai}
\usepackage{times}
\usepackage{graphicx}
\usepackage{latexsym}

\ecaisubmission   

\usepackage{amsmath}
\usepackage{amssymb}
\usepackage{amsthm}
\usepackage{framed}
\usepackage{graphicx}
\usepackage{color}
\usepackage{algorithm}
\usepackage[noend]{algorithmic}
\usepackage{url}

\usepackage{pifont}
\newcommand{\cmark}{\ding{51}}%
\newcommand{\xmark}{\ding{55}}%

\def\comment(#1){{\textcolor{blue}{\textbf{#1}}}}
\def\eql(#1,#2){{#1\!\!=\!#2}}

\newtheorem{thm}{Theorem}
\newtheorem{lemma}{Lemma}

\newtheorem{proposition}{Proposition}
\newtheorem{definition}{Definition}

\newcommand\shrink[1]{}

\def\n(#1){\bar{#1}}

\def\a{{\bf a}}

\def\X{{\bf X}}

\def\eql(#1,#2){{#1\!\!=\!#2}}

\def\eql(#1,#2){{#1\!=\!#2}}

\def\clap#1{\hbox to 0pt{\hss#1\hss}}

\newcommand\Amain[1]{
\textbf{main:}
#1
\vspace{1mm}
}

\def\Ca{{${\cal C}_1$}}
\def\Cb{{${\cal C}_2$}}
\def\Cd{{${\cal C}_3$}}

\def\n{{\sigma}}
\def\nn{{\overline{\sigma}}}
\def\t{{\tau}}
\def\nt{{\overline{\tau}}}
\def\a{{\rho}}
\def\na{{\overline{\rho}}}

\def\R{{\cal R}}
\def\cons{{\sf consensus}}
\def\filter{{\sf filter}}
\def\reason{{\sf reason}}

\def\ctd{{\sc c2d}}
\def\mctd{{\sc mini\_c2d}}
\def\df{{\sc d4}}

\begin{document}

\title{On The Reasons Behind Decisions\footnote{Will appear in proceedings of the European Conference on Artificial Intelligence (ECAI), Spain 2020.}}
\author{Adnan Darwiche\institute{University of California, Los Angeles,
email: darwiche@cs.ucla.edu}
 \and Auguste Hirth\institute{University of California, Los Angeles,
email: ahirth@cs.ucla.edu}}

\maketitle
\bibliographystyle{ecai}

\begin{abstract}
Recent work has shown that some common machine learning classifiers can be compiled into
Boolean circuits that have the same input-output behavior.
We present a theory for unveiling the {\em reasons} behind the decisions made by Boolean classifiers
and study some of its theoretical and practical implications. We define notions such as
sufficient, necessary and complete reasons behind decisions, in addition to classifier and decision bias.
We show how these notions can be used to evaluate counterfactual statements such 
as ``a decision will stick even if \ldots because \ldots\ .'' 
We present efficient algorithms for computing these notions, which are based on new advances
on tractable Boolean circuits, and illustrate them using a case study.
\end{abstract}

\section{Introduction}
\label{sec:intro}

Recent work has shown that some common machine learning classifiers can be compiled into Boolean circuits that
make the same decisions. This includes Bayesian network classifiers with discrete features~\cite{ChanD03,ShihCD19}
and some types of neural networks~\cite{ShihDC19b,ChoiShiShihDarwiche19}. Proposals were also extended to
explain and verify these numeric classifiers by operating on their compiled circuits~\cite{ShihCD18,ShihCD18b}.
We extend this previous work by proposing a theory for reasoning about the decisions made by classifiers
and discus its theoretical and practical implications.

In the proposed theory, a {\em classifier} is a Boolean function. Its variables are called {\em features,} a particular input
is called an {\em instance} and the function output on some instance is called a {\em decision.}
If the function outputs \(1\) on an instance, the instance and decision are said to be {\em positive;} 
otherwise, they are {\em negative.} 
Figure~\ref{fig:classifiers} depicts a classifier~(\Ca) for college admission,
represented as an Ordered Binary Decision Diagram (OBDD)~\cite{Bryant86}.
This OBDD was compiled from the Bayesian network (BN) classifier in Figure~\ref{fig:bnc} using the algorithm in~\cite{ShihCD19}.
The OBDD is guaranteed to make the same decision as the BN classifier on every instance (same input-output behavior).

Our main goal is to {\em explain} the decisions made by a classifier on specific instances by way of providing
various insights into what caused these decisions. 
Consider Susan who passed the entrance exam, is a first-time applicant, has no work experience and a high GPA. 
Susan will be admitted by classifier~\Ca\ depicted in Figure~\ref{fig:classifiers}. 
She also comes from a rich hometown and will be admitted by classifier~\Cb\ depicted in the same figure.
We can say that Susan was admitted by classifier~\Ca\ {\em because} she passed the 
entrance exam and has a high GPA.
We can also say that {\em one reason why} classifier~\Cb\ admitted Susan is that she passed the entrance exam
and has a high GPA (there are other reasons in this case). 
Moreover, we can say that classifier~\Cb\ {\em would still} admit Susan {\em even if} 
she did not have a high GPA {\em because} she passed the entrance exam and comes from a rich hometown. Finally, we can say that 
classifier~\Cb\ is biased as it can make biased decisions: ones that are based on {\em protected} features. 
For example, it will make different decisions on two applicants who have the same characteristics except 
that one comes from a rich hometown and the other does not. 
We will also show that one can sometimes prove classifier bias by inspecting the reasons behind one of its unbiased decisions.

We will give formal definitions and justifications for the statements exemplified above
and show how to compute them algorithmically. As far as semantics, 
the main tool we will employ is the classical notion of prime implicants~\cite{BooleanFunctions,quine1,mccluskey,quine2}. 
On the computational side, we will exploit tractable Boolean circuits~\cite{DarwicheJAIR02} while
providing some new fundamental results that further extend the reach of these circuits to computing explanations. 

This paper is structured as follows.
We review prime implicants in Section~\ref{sec:PIs}. We follow by introducing the notions of sufficient, 
necessary and complete reasons in Sections~\ref{sec:sufficient}--\ref{sec:necessary}. Counterfactual
statements about decisions are discussed in Section~\ref{sec:eib}, followed by a discussion of
decision and classifier bias in Section~\ref{sec:bias}. We dedicate Section~\ref{sec:compute} to
algorithms that compute the introduced notions while illustrating them using a case study in
Section~\ref{sec:study}. We finally close with some concluding remarks in Section~\ref{sec:conclusion}.

\begin{figure}[tb]
\centering
\includegraphics[height=0.2\textheight]{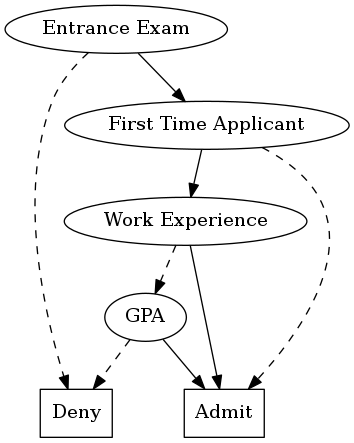}
\hspace{2mm}
\includegraphics[height=0.2\textheight]{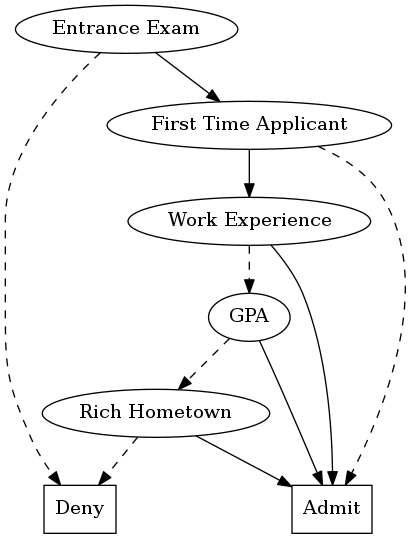}
\caption{OBDD \Ca\ (left) and OBDD \Cb\ (right). To classify an instance,
we start at the root OBDD node and repeat the following. If the feature
we are at is positive, we follow the solid edge, otherwise the dotted edge. 
\label{fig:classifiers}}
\includegraphics[width=0.40\textwidth]{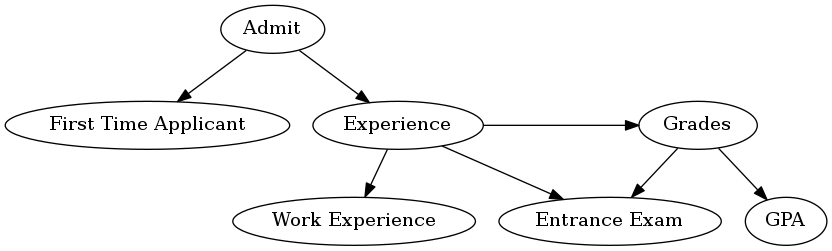}
\caption{The structure of a Bayesian network classifier. \label{fig:bnc}}
\end{figure}

\section{Classifiers, Decisions and Prime Implicants} 
\label{sec:PIs}

We represent a classifier by a propositional formula \(\Delta\) whose models (i.e., satisfying assignments)
correspond to positive instances. The negation of the formula characterizes negative instances.
Classifiers~\Ca\ and \Cb\ of Figure~\ref{fig:classifiers} are represented by the following formulas:
\begin{eqnarray*}
\Delta_1 & = & E \wedge (\neg F \vee G \vee W) \\
\Delta_2 & = & E \wedge (\neg F \vee G \vee W \vee R)
\end{eqnarray*}
We use \(\Delta(\alpha)\) to denote the decision (\(0\) or \(1\)) of classifier \(\Delta\) 
on instance \(\alpha\) (that is, \(\Delta(\alpha) =1\) iff \(\alpha \models \Delta\) and \(\Delta(\alpha) =0\) iff \(\alpha \models \neg \Delta\)).
We also define \(\Delta_\alpha = \Delta\) if the decision is positive and \(\Delta_\alpha = \neg \Delta\) 
if the decision is negative. This notation is critical and we use it frequently later noting that
\(\alpha \models \Delta_\alpha\) and 
\(\Delta(\alpha) = \Delta(\beta)\) iff \(\Delta_\alpha = \Delta_\beta\).
 
A {\em literal} is a variable (positive literal) or its negation (negative literal). A {\em term} is a consistent
conjunction of literals. Term \(\t_i\) {\em subsumes} term \(\t_j\), written \(\t_j \models \t_i\), iff
\(\t_j\) includes the literals of \(\t_i\).
For example, term \(E \wedge \neg F\) subsumes term \(E \wedge \neg F \wedge G\).
We treat a term as the {\em set} of its literals so we may write \(\t_i \subseteq \t_j\) to also mean
that \(\t_i\) subsumes \(\t_j\). We sometimes refer to a literal as a {\bf characteristic} 
and to a term \(\t\) as a {\bf property} (of an instance). We use \(\nt\) to denote the property 
resulting from negating every characteristic in property \(\t\).
We sometimes use a comma (\(,\)) instead 
of a conjunction (\(\wedge\)) when describing properties and instances (e.g., \(E,\neg F\) instead of \(E \wedge \neg F\)).

An {\em implicant} \(\t\) of propositional formula \(\Delta\) is a term that satisfies \(\Delta\), written \(\t \models \Delta\).
A {\em prime implicant} is an implicant that is not subsumed by any other implicant. 
For example, \(E \wedge \neg F \wedge G\) is an implicant of \(\Delta_1\) but is not prime 
since it is subsumed by another implicant \(E \wedge \neg F\), which happens to be prime.
Classifier~\Ca\ has the following prime implicants:
\begin{eqnarray*}
\Delta_{1} & : & (E \wedge \neg F)\:\: (E \wedge G)\:\: (E \wedge W) \\ 
\neg \Delta_{1} & : & (\neg E) \:\: (F \wedge \neg G \wedge \neg W) 
\end{eqnarray*}
Classifier~\Cb\ has the following prime implicants:
\begin{eqnarray*}
\Delta_{2} & : & (E \wedge \neg F)\:\: (E \wedge G)\:\: (E \wedge W)\:\: (E \wedge R) \\
\neg \Delta_{2} & : & (\neg E) \:\: (F \wedge \neg G \wedge \neg W \wedge \neg R)
\end{eqnarray*}

\shrink{
The following lemma is key to a number of results we present later.
\begin{lemma}\label{lem:PI}
A term \(\t\) cannot be an implicant of both \(\Delta\) and \(\neg \Delta\).
\end{lemma}
\begin{proof}
If term \(\t\) is an implicant of both \(\Delta\) and \(\neg \Delta\), then \(\t\) is inconsistent.
This is impossible by definition of a term.
\end{proof}
}

The set of prime implicants for a propositional formula can be quite large, which motivated the notion
of a {\em prime implicant cover}~\cite{quine1,mccluskey,quine2}. A set of terms \(\t_1, \ldots, \t_n\) is prime implicant cover 
for propositional formula \(\Delta\) if each term \(\t_i\) is a prime implicant of \(\Delta\) and \(\t_1 \vee \ldots \vee \t_n\) is 
equivalent to \(\Delta\). A cover may not include all prime implicants, with the missing ones called {\em redundant.} 
While covers can be useful computationally, they may not always be appropriate for explaining 
classifiers as they may lead to incomplete explanations (more on this later).

\shrink{
We used the \textit{primer} tool for computing prime implicants in our experiments~\cite{joao1,joao2},
which generates all prime implicants not just covers. This tool can also accept circuit representations
of propositional formula, as opposed to only CNF representations.
}

We will make use of the {\em conditioning} operation on propositional formula.
To condition formula \(\Delta\) on literals \(\t\), denoted \(\Delta \vert \t\),
is to replace every literal \(l\) in \(\Delta\) with \(1\) if \(l \in \t\) and with \(0\) if \(\neg l \in \t\).
We will also use {\em existential quantification:} \(\exists X \Delta = (\Delta \vert X) \vee (\Delta \vert \neg X)\).

In the next few sections, we introduce the notions of sufficient, complete and necessary reasons behind
a decision. We use these notions later to define decision and classifier bias in addition to giving semantics
to counterfactual statements relating to decisions.

\section{Sufficient Reasons} 
\label{sec:sufficient}

Prime implicants have been studied and utilized extensively in the AI and computer 
science literature.\footnote{One classical application of prime implicants in AI has been in the area of model-based 
diagnosis, where they have been used to formalize the notion of {\em kernel diagnoses} ~\cite{diagnosisPI}.
A kernel diagnosis is defined for a given device behavior and is a minimal term representing the health of some device components. 
Any system state that is compatible with a kernel diagnosis is feasible under the given system behavior. Moreover, the
set of kernel diagnoses characterize all feasible system states under the given behavior.}
However, their active utilization in explaining decisions is more recent, e.g.,~\cite{ShihCD18,IgnatievNM19,JoaoNIPS19,ReasonsMoral},
and introduced a key connection to properties of instances that we highlight next and exploit computationally later.

\begin{definition}[\bf Sufficient Reason~\cite{ShihCD18}]
\label{def:sufficient}
A sufficient reason for decision \(\Delta(\alpha)\) is a property of instance \(\alpha\) that is also a prime 
implicant of \(\Delta_\alpha\) (recall \(\Delta_\alpha\) is \(\Delta\) if the decision is positive and \(\neg \Delta\) otherwise).
\end{definition}
 
A sufficient reason identifies characteristics of an instance that justify the decision:
The decision will stick even if other characteristics of the instance were different. 
A sufficient reason is minimal:
None of its strict subsets can justify the decision. A decision can have multiple sufficient reasons,
sometimes a very large number of them.\footnote{The LIME~\cite{LIME} and Anchor~\cite{ANCHOR} systems 
can be viewed as computing approximations of sufficient reasons. The quality of these approximations has been 
evaluated on some datasets and corresponding classifiers in~\cite{JoaoApp}, where an approximation is called
{\em optimistic} if it is a strict subset of a sufficient reason and  {\em pessimistic} if it is a strict superset
of a sufficient reason.}

There is a key difference between prime implicants and sufficient reasons:
the latter must be properties of the given instance. This has significant computational implications
that we exploit in Section~\ref{sec:compute}.

Sufficient reasons were introduced in~\cite{ShihCD18} under the name of {\em PI-explanations.}
The new name we adopt is motivated by further distinctions that we
draw later and was also used in~\cite{ReasonsMoral}. 
We will also sometimes say ``a reason'' to mean ``a sufficient reason.''

Greg passed the entrance exam, is not a first time applicant, does not have a high GPA
but has work experience (\(\alpha = E, \neg F, \neg G, W\)).
Classifier~\Ca\ admits Greg, a decision that can be explained using either of the
following sufficient reasons:
\begin{itemize}
\item Passed the entrance exam and is not a first time applicant (\(E, \neg F\)).
\item Passed the entrance exam and has work experience (\(E, W\)).
\end{itemize}
Since Greg passed the entrance exam and has applied before, he will be admitted
even if his other characteristics were different. Similarly, since Greg passed
the entrance exam and has work experience, he will be admitted even if his other characteristics were different.

\begin{proposition}\label{prop:sufficient1}
Every decision has at least one sufficient reason.
\end{proposition}
\begin{proof}
Consider decision \(\Delta(\alpha)\). We have \(\alpha \models \Delta_\alpha\),
which means \(\Delta_\alpha\) is consistent and must have at least
one prime implicant (the empty term if \(\Delta_\alpha\) is valid). Moreover, at least one
of these prime implicants must be a property of instance \(\alpha\) since 
\(\alpha \models \Delta_\alpha\) and since \(\Delta_\alpha\) is equivalent to the disjunction 
of its prime implicants. Hence, we have at least one sufficient reason for the decision.
\end{proof}

A classifier may make the same decision on two instances but for different reasons (i.e., disjoint sufficient reasons). 
However, if two decisions on distinct instances share a reason, they must be equal.
\begin{proposition} \label{prop:sufficient2}
If decisions \(\Delta(\alpha)\) and \(\Delta(\beta)\) share a sufficient reason,
the decisions must be equal \(\Delta(\alpha) = \Delta(\beta)\).
\end{proposition}
\begin{proof}
Suppose the decisions share sufficient reason \(\t\). Then \(\t\) is property of both \(\alpha\) and \(\beta\)
and \(\t\) is a prime implicant of both \(\Delta_\alpha\) and \(\Delta_\beta\). 
Hence, \(\Delta_\alpha = \Delta_\beta\) since \(\t\) is consistent and \(\Delta(\alpha) = \Delta(\beta)\).
\end{proof}

\shrink{
The following proposition highlights a critical property of sufficient reasons. 
A reason not only triggers the decision but must also be a property of the instance under consideration. 
While this second property is to be expected, it can sometimes be missed.

\begin{proposition} \label{prop:sufficient3}
If decisions \(\Delta(\alpha)\) and \(\Delta(\beta)\) do not share sufficient reasons,
then no sufficient reason for decision \(\Delta(\alpha)\) can be a property of 
instance \(\beta\) (even if the decisions are equal).
\end{proposition}
\begin{proof}
Let \(\t\) be a sufficient reason for decision \(\Delta(\alpha)\). Then \(\t\) is a prime 
implicant of \(\Delta_\alpha\). If \(\t\) is a property of instance \(\beta\), then 
\(\beta \models \t \models \Delta_\alpha\) and \(\Delta_\beta = \Delta_\alpha\).
Hence, \(\t\) is a sufficient reason for decision \(\Delta(\beta)\), a contradiction,
so \(\t\) cannot be a property of instance \(\beta\).
\end{proof}
}

We will see later that sufficient reasons can provide insights about a
classifier that go well beyond explaining its decisions.

\section{Complete Reasons}
\label{sec:the}

A sufficient reason identifies a minimal property of an instance that can 
trigger a decision. The {\em complete reason} behind a decision characterizes all
properties of an instance that can trigger the decision.

\begin{definition}[\bf Complete Reason]
\label{def:the}
The complete reason for a decision is the disjunction of all its sufficient reasons.
\end{definition}

\shrink{
\begin{proposition}\label{prop:the}
Let \(\R\) be the reason behind decision \(\Delta(\alpha)\).
Then \(\R \models \Delta_\alpha\).
\end{proposition}
\begin{proof}
Follows directly from Definition~\ref{def:the} and since sufficient reasons are implicants of \(\Delta_\alpha\).
\end{proof}
}

The complete reason for decision \(\Delta(\alpha)\) captures {\em every}
property of instance \(\alpha\), and {\em only} properties of instance \(\alpha\), that can trigger the decision.
It precisely captures why the particular decision is made.

\begin{thm}\label{theo:the}
Let \(\R\) be the complete reason for decision \(\Delta(\alpha)\). 
If instance \(\beta\) does not satisfy \(\R\) and \(\Delta(\beta) = \Delta(\alpha)\), 
then no sufficient reason for decision \(\Delta(\beta)\) can be a property of instance \(\alpha\).
\end{thm}
\begin{proof}
Suppose \(\beta \not \models \R\) and \(\Delta(\beta) = \Delta(\alpha)\).
Then \(\Delta_\beta = \Delta_\alpha\). Let \(\t\) be a sufficient reason for decision \(\Delta(\beta)\).
Then \(\t\) is a property of instance \(\beta\) and a prime implicant of both \(\Delta_\beta\) and \(\Delta_\alpha\).
If \(\t\) were a property of instance \(\alpha\), then \(\t\) is a sufficient reason for decision \(\Delta(\alpha)\),
\(\t \models \R\) and \(\beta \models \t \models \R\), a contradiction.
Hence, \(\t\) cannot be a property of instance \(\alpha\). 
\end{proof}
We will sometimes say ``the reason'' to mean ``the complete reason.''
Classifier~\Ca\ admits Greg (\(\alpha = E, \neg F, \neg G, W\)) for the reason \(\R = E \wedge (\neg F \vee W)\).
Greg was admitted because he passed the entrance exam and satisfied one of two additional 
requirements: he applied before and has work experience.
Classifier~\Ca\ also admits Susan (\(\beta = E, F, G, \neg W\)). Susan does not satisfy the reason \(\R\).
There is one sufficient reason for admitting Susan: she passed the entrance exam and has a good GPA (\(E,G\)),
which is not a property of Greg. The classifier admitted Greg and Susan for different reasons.

The complete reason for a decision is unique up to logical equivalence and can be used 
to enumerate its sufficient reasons.\footnote{Pierre Marquis observed that the complete reason can be formulated
using the notion of {\em literal forgetting} which is a more fine grained notion than {\em variable forgetting}
(also known as existential quantification)~\cite{Marquis00,DBLP:journals/jair/LangLM03,DBLP:journals/tocl/HerzigLM13}.}

\begin{thm}\label{theo:the-pi}
Let \(\R\) be the complete reason for decision \(\Delta(\alpha)\). 
The prime implicants of \(\R\) are the sufficient reasons for decision \(\Delta(\alpha)\).
\end{thm}
\begin{proof}
Let \(\t_1, \ldots, \t_n\) be the sufficient reasons for decision \(\Delta(\alpha)\)
and hence \(\R = \t_1 \vee \ldots \vee \t_n\).
The key observation is that terms \(\t_i\) are properties of instance \(\alpha\).
Hence, for every two terms \(\t_i\) and \(\t_j\), term \(\t_i\)
cannot contain some literal \(X\) while term \(\t_j\) containing literal \(\neg X\). 
The DNF \(\t_1 \vee \ldots \vee \t_n\) is then
closed under consensus.\footnote{The consensus rule infers the term \(\delta_1 \wedge \delta_2\) 
from terms \(X \wedge \delta_1\) and \(\neg X \wedge \delta_2\). One can convert a DNF 
into its set of prime implicants by closing the DNF under consensus and then removing 
subsumed terms; see~\cite[Chapter 3]{BooleanFunctions}.} 
Since no term \(\t_i\) subsumes another term \(\t_j\), the DNF  \(\t_1 \vee \ldots \vee \t_n\) contains
all prime implicants of \(\R\). Hence, the prime implicants of complete reason \(\R\)
are precisely the sufficient reasons of decision \(\Delta(\alpha)\).
\end{proof}

\section{Necessary Properties and Reasons}
\label{sec:necessary}

The {\em necessary property} of a decision is a maximal property of an instance that is essential 
for explaining the decision on that instance.

\begin{definition}[\bf Necessary Characteristics and Properties]
\label{def:necessary-p}
A characteristic is necessary for a decision iff it appears in every sufficient reason for the decision.
The necessary property for a decision is the set of all its necessary characteristics.
\end{definition}
The necessary property is unique but could be empty (when the decision has no necessary characteristics).

If an instance ceases to satisfy one necessary characteristic, the corresponding decision is guaranteed to change.
\begin{proposition}\label{prop:necessary1}
If instance \(\beta\) disagrees with instance \(\alpha\) on only one characteristic necessary for
decision \(\Delta(\alpha)\), then \(\Delta(\alpha) \neq \Delta(\beta)\).
\end{proposition}
\begin{proof}
Suppose \(\alpha\) and \(\beta\) are as premised. 
If \(\Delta(\alpha) = \Delta(\beta)\) then  \(\Delta_\alpha = \Delta_\beta\) and  \(\t = \alpha \cap \beta\) 
is an implicant of \(\Delta_\alpha\) by consensus on the flipped characteristic \(\a\).
Moreover, \(\t\) does not contain characteristic \(\a\) so it cannot be necessary, a contradiction.
\end{proof}

If an instance ceases to satisfy more than one necessary characteristic, the decision does not 
necessarily change. However, if the decision sticks then it would be for completely different reasons.

\begin{thm}\label{theo:necessary}
Let \(\beta\) be an instance that disagrees with instance \(\alpha\) on at least one characteristic
necessary for decision \(\Delta(\alpha)\). Decisions \(\Delta(\alpha)\) and \(\Delta(\beta)\) 
must have disjoint sufficient reasons. 
\end{thm}
\begin{proof}
Let \(\n\) be the necessary characteristics of decision \(\Delta(\alpha)\) that instances \(\alpha\) and \(\beta\) disagree on.
A sufficient reason \(\t\) of \(\Delta(\alpha)\) cannot be a property of instance \(\beta\) 
since \(\n \subseteq \t\) and \(\beta\) contains \(\nn\). 
Hence, \(\t\) cannot be a sufficient reason for decision \(\Delta(\beta)\) and the two
decisions must have disjoint sufficient reasons.
\end{proof}

Consider a classifier \(\Delta = (X\wedge Y \wedge Z) \vee (\neg X \wedge \neg Y \wedge Z)\)
and instance \(\alpha = X,Y,Z\). The decision \(\Delta(\alpha)\) is positive with \(X,Y,Z\) as the only 
sufficient reason. Hence, all three characteristics of \(\alpha\) are necessary:
Flipping any single characteristic of instance \(\alpha\) will lead to a negative decision. However,
flipping the two characteristics \(X\) and \(Y\) preserves the positive decision but leads to a new, single
sufficient reason \(\neg X, \neg Y, Z\). 

The complete reason for a decision has enough information to compute its necessary 
characteristics and necessary property.

\begin{proposition}\label{prop:necessary2}
A characteristic is necessary for a decision iff it is implied by the decision's complete reason.
\end{proposition}
\begin{proof}
Follows from Definition~\ref{def:necessary-p} and Theorem~\ref{theo:the-pi}.
\end{proof}

We can now define the notion of necessary reason.
\begin{definition}[\bf Necessary Reason]
\label{def:necessary-r}
The necessary property of a decision is called the necessary reason for the decision iff it is
the only sufficient reason for the decision.
\end{definition}
There may be no necessary reason for a  decision as there may be no instance property that is
both sufficient and necessary for triggering the decision. We next highlight how the complete
reason for a decision, being a {\em condition} on an instance instead of a {\em property,} 
is always necessary and sufficient for explaining the decision.

Consider the complete reason \(\R\) for decision \(\Delta(\alpha)\) and recall that
it characterizes all properties of instance \(\alpha\) that can trigger 
the decision: \(\R \equiv \bigvee_{\t \models \Delta_\alpha} \t,\) where \(\t\) is a property of instance \(\alpha\). 
The reason \(\R\) is then a logical condition that triggers the decision (\(\R \models \Delta_\alpha\)). 
If the complete reason is weakened into a condition \(\R_w\) that continues to trigger the 
decision (\(\R \models \R_w \models \Delta_\alpha\)), then \(\R_w\) will admit properties not satisfied by instance \(\alpha\). 
Moreover, if it is strengthened into a condition \(\R_s\), then \(\R_s\) will continue to trigger the decision (\(\R_s \models \R \models \Delta_\alpha\))
but will stop admitting some properties of instance \(\alpha\) that can trigger the decision. 
Hence, the complete reason \(\R\) is a necessary and sufficient condition (not necessarily a property) 
for explaining the decision on instance \(\alpha\).

\section{Decision Counterfactuals}
\label{sec:eib}

We mentioned Susan earlier who passed the entrance exam, is a first time applicant, 
has a high GPA but no work experience (\(\alpha = E, F, G, \neg W\)).
Classifier~\Ca\ admits Susan {\em because} she passed the entrance exam and has a 
high GPA. Greg was also admitted by this classifier. His application is similar to Susan's
except that he applied before and has work experience (\(\beta = E, \neg F, G, W\)). 
We cannot pinpoint a single property of Greg that triggered admission, so we cannot issue 
a ``because'' statement when explaining this decision.  

\begin{definition}[\bf Because]
\label{def:because}
Consider decision \(\Delta(\alpha)\) and let \(\t\) be a property of instance \(\alpha\).
The decision is made ``because \(\t\)'' iff \(\t\) is the complete reason for the decision.
\end{definition}

\begin{proposition}\label{prop:because}
A decision is made because \(\t\) iff \(\t\) is the necessary reason for the 
decision (i.e., the only sufficient reason).
\end{proposition}
\begin{proof}
Follows from Definitions~\ref{def:sufficient},~\ref{def:the} and~\ref{def:necessary-r}.
\end{proof}

One may be interested in statements that provide insights into a decision
beyond the reasons behind it. For example, how the classifier may have decided if 
some instance characteristics were different. 

An example statement is the one we mentioned in Section~\ref{sec:intro}: 
Susan would have been admitted even if she did not have a high GPA because 
she comes from a rich hometown and passed the entrance exam. 
This statement exemplifies counterfactuals of the following form: The decision will 
stick even if \(\na\) because \(\t\), where \(\a\) and \(\t\) are properties of the given instance. 

\begin{definition}[\bf Even-If-Because]
\label{def:eib}
Consider decision \(\Delta(\alpha)\) and let \(\a\) and~\(\t\) be properties of instance \(\alpha\).
The decision sticks ``even if \(\na\) because \(\t\)'' iff \(\t\) is the complete reason for the decision 
after changing property \(\a\) of instance \(\alpha\) to \(\na\) (i.e., flipping all characteristics in \(\a\)).
\end{definition}

Let \(\beta\) be the result of replacing property \(\a\) of instance \(\alpha\) by \(\na\)
and suppose that \(\t\) is the complete reason for decision \(\Delta(\beta)\). Then \(\t\) is the only sufficient
reason for decision \(\Delta(\beta)\) by Definition~\ref{def:the}. Hence \(\beta \models \t\) and
properties \(\a\) and \(\t\) must be disjoint. Moreover, \(\alpha \models \t \models \Delta_\beta\) so
\(\Delta_\alpha=\Delta_\beta\) and \(\Delta(\alpha) = \Delta(\beta)\). Hence, the decision sticks 
``even if \(\na\) because \(\t\).''

Applicant Susan discussed earlier (\(\alpha = E, F, G, \neg W, R\)) is admitted by classifier~\Cb. 
The decision will stick even if Susan had a low GPA (\(\neg G\)) because she comes from a rich hometown 
and passed the entrance exam (\(E,R\)).
This statement is justified since \(E,R\) is the complete reason for decision \(\Delta(\beta)\) where 
\(\beta = E, F, \neg G, \neg W, R\) is the result of replacing characteristic \(G\) by \(\neg G\) in instance \(\alpha\).

Jackie did not pass the entrance exam, is not a first time
applicant, has a low GPA but has work experience (\(\alpha = \neg E, \neg F, \neg G, W\)).
Jackie is denied admission by classifier~\Ca. The decision will stick even if Jackie had a high GPA (\(G\))
because she did not pass the entrance exam (\(\neg E\)). This statement is justified since 
\(\neg E\) is the complete reason for decision \(\Delta(\beta)\)  where \(\beta = \neg E, \neg F, G, W\)
is the result of replacing characteristic \(\neg G\) by \(G\) in instance \(\alpha\).

\section{Decision Bias and Classifier Bias}
\label{sec:bias}
We will now discuss the dependence of decisions on certain features, with a particular
application to detecting decision and classifier bias.

Intuitively, a decision is {\em biased} if it depends on a {\em protected feature:} one that should 
not be used when making the decision (e.g., gender, zip code, or ethnicity).\footnote{A protected feature
may have been unprotected during classifier design.}
We formalize bias next while making a distinction between classifier bias and
decision bias: A classifier may be biased in that it could make biased decisions, but 
the particular decisions it already made may have been unbiased. 
While classifier bias can always be detected by examining its decision function,
we will show that it can sometimes be detected by examining the
complete reason behind one of its unbiased decisions.

\begin{definition}[\bf Decision Bias]
\label{def:bias-d}
Decision \(\Delta(\alpha)\) is biased iff \(\Delta(\alpha) \neq \Delta(\beta)\)
for some \(\beta\) that disagrees with \(\alpha\) on only protected features.
\end{definition}

Bias can be positive or negative. For example, an applicant may be 
admitted because they come from a rich hometown, or may be denied
admission because they did not come from a rich hometown.

The following result provides a necessary and sufficient condition for detecting decision bias.
\begin{thm}\label{theo:bias-d}
A decision is biased iff each of its sufficient reasons contains at least one protected feature.
\end{thm}
\begin{proof}
Suppose decision \(\Delta(\alpha)\) is biased yet has a sufficient reason \(\t\) with no protected features.
We will now show a contradiction. Since the decision is biased, there must exist an instance \(\beta\)
that disagrees with instance \(\alpha\) on only protected features and  \(\Delta(\alpha) \neq \Delta(\beta)\).
Since \(\t\) is a property of \(\alpha\) and \(\beta\), we have \(\alpha \models \t \models \Delta_\alpha\) and
\(\beta \models \t \models \Delta_\alpha\). Hence, \(\Delta_\alpha = \Delta_\beta\) and
\(\Delta(\alpha) = \Delta(\beta)\), which is a contradiction.

Suppose every sufficient reason of decision \(\Delta(\alpha)\) contains at least one protected feature. Let \(\X\) be
these protected features and \(\t\) be the characteristics of instance \(\alpha\) that do not involve features \(\X\). 
Assume \(\Delta(\alpha) = \Delta(\beta)\) for every instance \(\beta\) that agrees with instance
\(\alpha\) on characteristics \(\t\) (that is, \(\beta\) disagrees with \(\alpha\) only on features in \(\X\)).
Term \(\t\) must then be an implicant of \(\Delta_\alpha\) and a subset \(\n\) of \(\t\) must be a prime implicant 
of \(\Delta_\alpha\) (could be \(\t\) itself). Since \(\t\) is a property of instance \(\alpha\), decision \(\Delta(\alpha)\) 
has sufficient reason \(\n\) that does not include a protected feature in \(\X\), which is a contradiction. 
Hence, \(\Delta(\alpha) \neq \Delta(\beta)\) for some instance \(\beta\) that disagrees with instance \(\alpha\) 
on only protected features in \(\X\), and decision \(\Delta(\alpha)\) is biased.
\end{proof}

Theorem~\ref{theo:bias-d} does not require sufficient reasons to share a protected feature,
only that each must contain at least one protected feature.

Consider classifier \Cd, which admits applicants who have a good GPA (\(G\)) as long as they
pass the entrance exam (\(E\)), are male (\(M\)) or come from a rich hometown (\(R\)):
\begin{equation}
\Delta_3 = (G \wedge E) \vee (G \wedge M) \vee (G \wedge R). \label{class:d}
\end{equation}
Bob has a good GPA, did not pass the entrance exam and comes from a rich hometown 
(\(\alpha = G, \neg E, M, R\)). He is admitted with two sufficient reasons: \(G, M\) and \(G,R\).
The decision is biased since each sufficient reason contains a protected feature.
This classifier will not admit Nancy who has similar characteristics but does not come from a rich
hometown: \(\beta = G, \neg E, \neg M, \neg R\). It will also admit Scott who has the same
characteristics as Nancy: \(\gamma = G, \neg E, M, \neg R\). 

Even though this classifier is biased, some of its decisions may be unbiased. If an applicant has a good
GPA and passes the entrance exam (\(G, E\)), they will be admitted regardless of their protected 
characteristics. Moreover, if an applicant does not have a good GPA (\(\neg G\)), they will be denied
admission regardless of their other characteristics, including protected ones.

\begin{definition}[\bf Classifier Bias]
\label{def:bias-c}
A classifier is biased iff at least one of its decisions is biased.
\end{definition}
A classifier may be biased, but some of its decisions may be unbiased. 
Moreover, one can sometimes infer classifier bias by inspecting the sufficient reasons behind one of its unbiased decisions. 

\begin{thm}\label{theo:bias-c}
A classifier is biased iff one of its decisions has a sufficient reason that includes a protected feature.
\end{thm}
\begin{proof}
Suppose classifier \(\Delta\) is biased. By Definition~\ref{def:bias-c}, some decision \(\Delta(\alpha)\) is biased.
By Theorem~\ref{theo:bias-d}, every sufficient reason of decision \(\Delta(\alpha)\) must contain at least one
protected feature. 

Suppose some decision \(\Delta(\alpha)\) has a sufficient reason \(\t\) that contains protected features \(\X \neq \emptyset\).
For any instance \(\beta\) such that \(\beta \models \t\), we must have \(\Delta(\beta) = \Delta(\alpha)\). 
We now show that there is an instance \(\beta \models \t\) and instance \(\gamma\) that disagrees with \(\beta\) 
on only features \(\X\) such that \(\Delta(\beta) \neq \Delta(\gamma)\).
Suppose the contrary is true: for all such \(\beta\) and \(\gamma\), we have \(\Delta(\beta)=\Delta(\gamma)=\Delta(\alpha)\).
Then \(\t \setminus \a\) is an implicant of \(\Delta_\alpha\), where \(\a\) are the protected characteristics in \(\t\).
This is impossible since \(\t\) is a prime implicant of \(\Delta_\alpha\). Hence, \(\Delta(\beta) \neq \Delta(\gamma)\)
for some \(\beta\) and \(\gamma\) with the stated properties and the classifier is biased.
\end{proof}

If decision \(\Delta(\alpha)\) has protected features in some but not all of its sufficient reasons, 
the decision is not biased according to Theorem~\ref{theo:bias-d}. But classifier \(\Delta\) is biased 
according to Theorem~\ref{theo:bias-c} as we can {\em prove} that it will make a biased decision on some 
other instance \(\beta \neq \alpha\).

Consider classifier \Cd\ in~(\ref{class:d}) and Lisa who has a good GPA, passed the entrance exam
and comes from a rich hometown (\(G, E, \neg M, R\)). The classifier will admit Lisa for
two sufficient reasons: \(G,E\) and \(G,R\). The decision is unbiased: any applicant who has similar unprotected
characteristics will be admitted. However, since one of the sufficient reasons contains a protected
feature, the classifier is biased as it can make a biased decision on a different applicant. 
The proof of Theorem~\ref{theo:bias-c} suggests that the classifier will make different decisions on
two applicants with a good GPA that disagree only on whether they come from a rich hometown.
Nancy (\(G, \neg E, \neg M, \neg R\)) and  Heather (\(G, \neg E, \neg M, R\)) are such applicants.

The following theorem shows how one can detect decision bias using the complete reason
behind the decision. We use this theorem (and Theorem~\ref{theo:bias-c-cond}) when discussing 
algorithms in Section~\ref{sec:compute}.

\begin{thm}\label{theo:bias-d-exist}
A decision is biased iff \(\exists (X_1, \ldots, X_n) \R\) is not valid where \(X_1, \ldots, X_n\)
are all unprotected features and \(\R\) is the complete reason behind the decision.
\end{thm}
\begin{proof}
Let \(\t_1, \ldots, \t_n\) be the decision's sufficient reasons and hence 
\(\R = \t_1 \vee \ldots \vee \t_n\). Existentially quantifying variables \(X_i\)
from a DNF is done by replacing their literals with \(1\). The result is valid
iff some term \(\t_i\) contains only variables in \(X_1, \ldots, X_n\). Hence, 
\(\exists X_1, \ldots, X_n \R\) is not valid iff each term \(\t_i\) contains
variables beyond  \(X_i\) (i.e., each sufficient reason contains protected features).
\end{proof}

The following result shows how classifier bias can sometimes be detected based
on the complete reason behind an unbiased decision.

\begin{thm}\label{theo:bias-c-cond}
A classifier is biased if \(\R|X \not \equiv \R|\neg X\) where \(X\) is a protected feature 
and \(\R\) is the complete reason for some decision. 
\end{thm}
\begin{proof}
Given Theorems~\ref{theo:the-pi} and~\ref{theo:bias-c}, it is sufficient to show that
\(\R|X \not \equiv \R|\neg X\) iff feature \(X\) appears in some prime implicant of \(\R\).
Let \(\t_1, \ldots, \t_n\) be the prime implicants of \(\R\).
Feature \(X\) appears either positively or negatively in these prime implicants
since terms \(\t_i\) are all properties of the same instance.
Suppose without loss of generality that feature \(X\) appears positively in terms \(\t_i\) (if any).
Then \(\R|X \equiv \bigvee_{X \not \in \t_i} \t_i \vee \bigvee_{X \in \t_i} \t_i \setminus \{X\}\) 
and \(\R|\neg X \equiv \bigvee_{X \not \in \t_i} \t_i\).
Hence \(\R|X \not \equiv \R|\neg X\) iff \(X \in \t_i\) for some prime implicant \(\t_i\).
\end{proof}
Theorem~\ref{theo:bias-c-cond} follows from Theorems~\ref{theo:the-pi} and~\ref{theo:bias-c} 
and a known result: A Boolean function depends on a variable \(X\) iff \(X\) 
appears in one of its prime implicants. We include the full proof for completeness.

\section{Computing Reasons and Related Queries}
\label{sec:compute}

\begin{figure*}[tb]
\centering
\includegraphics[height=0.20\textheight]{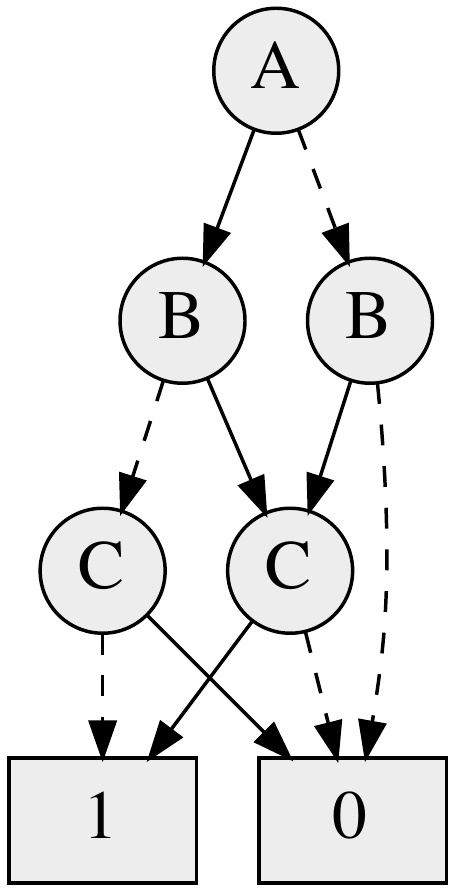}
\hspace{2mm}
\includegraphics[height=0.30\textheight]{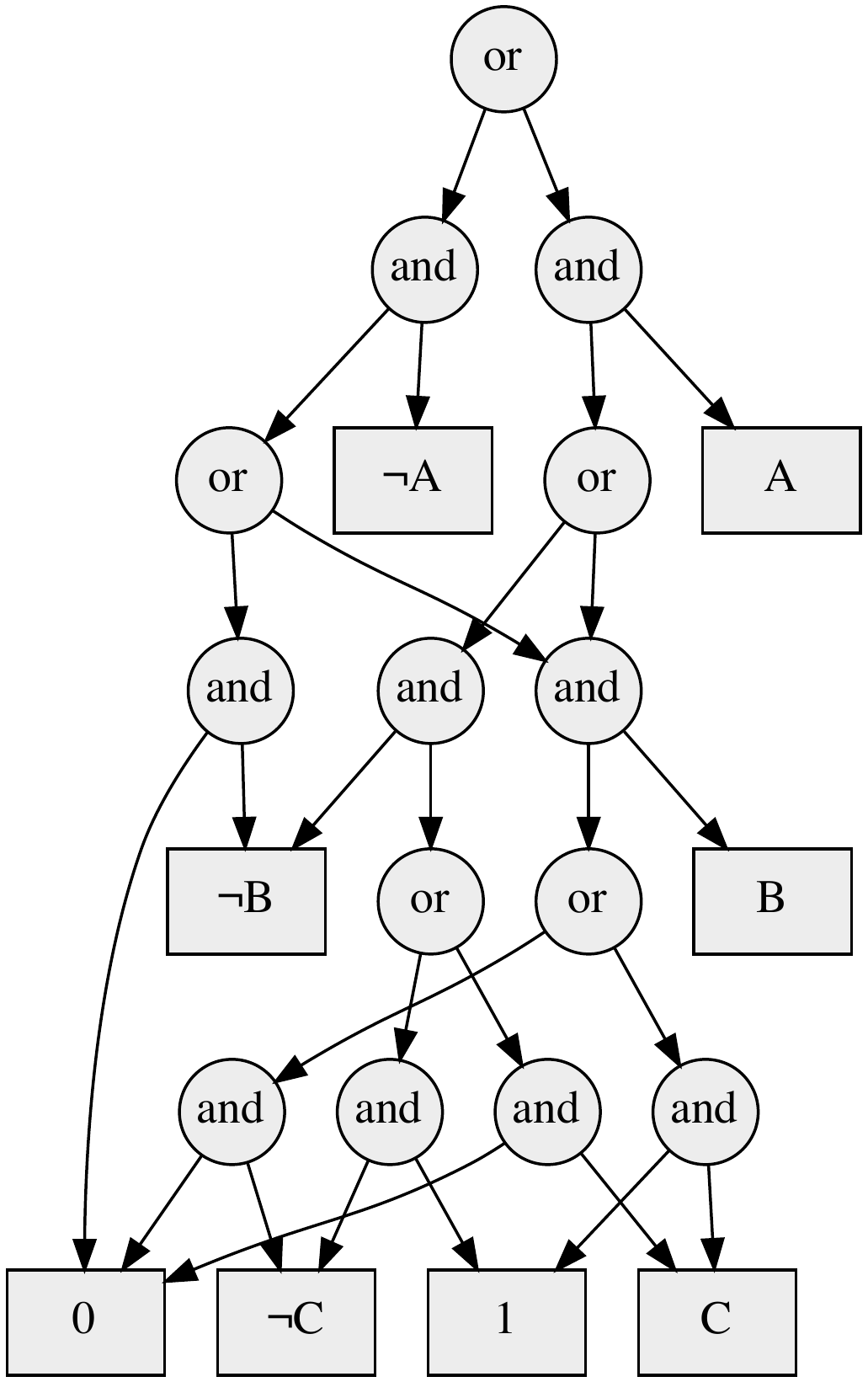}
\hspace{2mm}
\includegraphics[height=0.30\textheight]{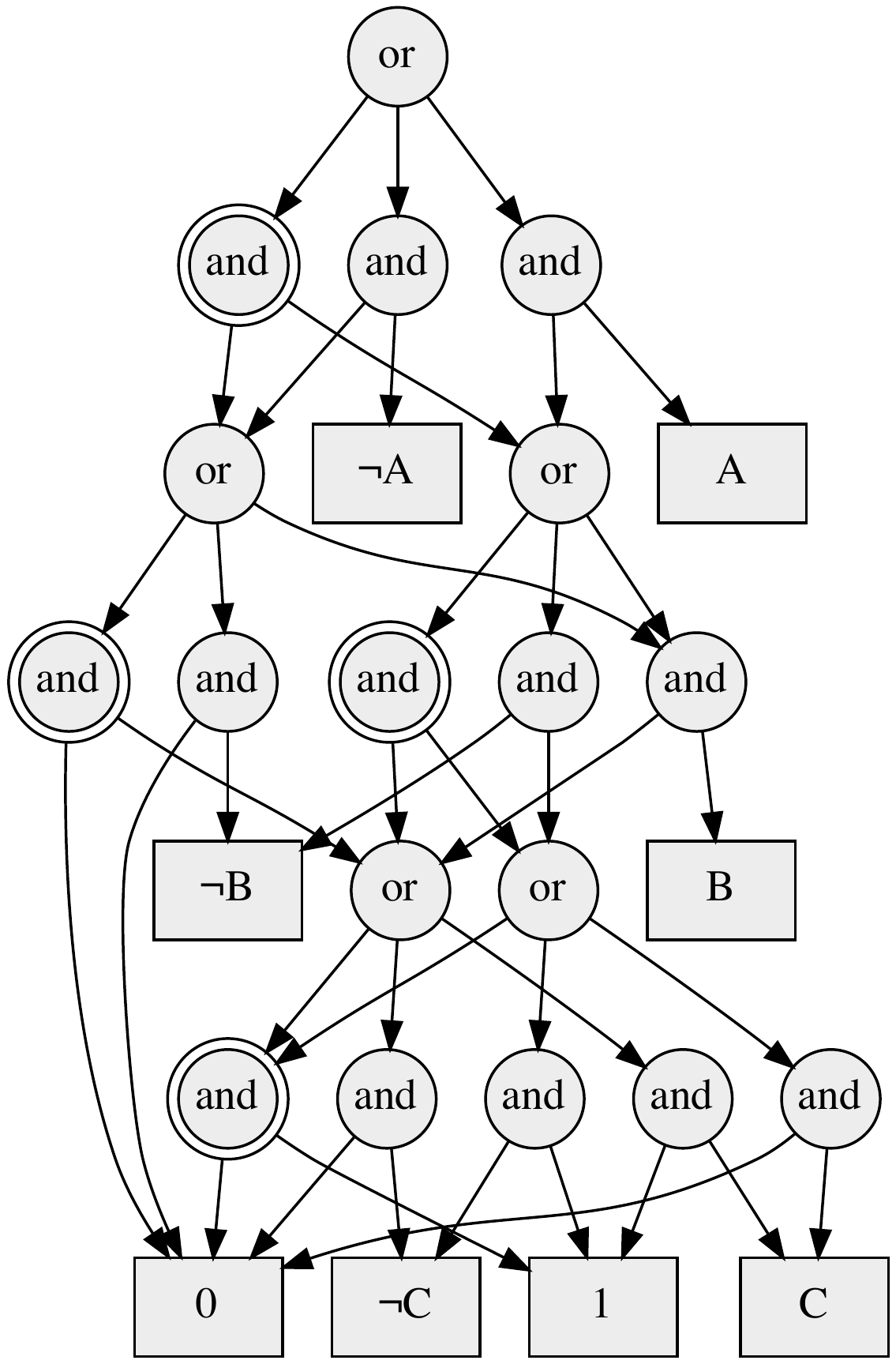}
\hspace{2mm}
\includegraphics[height=0.30\textheight]{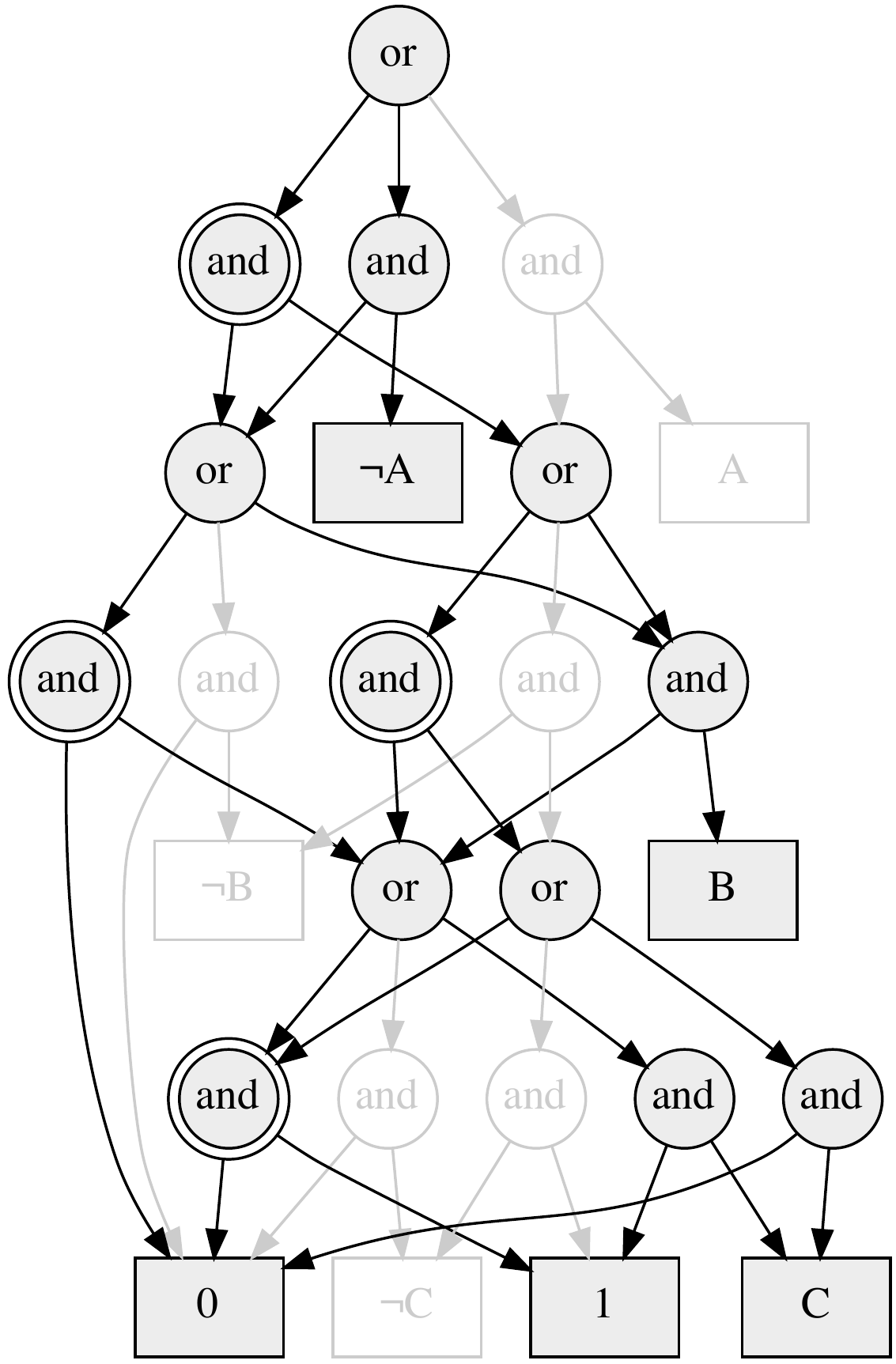}
\caption{From left to right: OBDD, Decision-DNNF circuit, consensus circuit, 
and the filtering of consensus circuit by instance \(\neg A,B,C\).
\label{fig:obdd2nnf}}
\end{figure*}

The enumeration of PI-explanations (sufficient reasons) was treated in~\cite{ShihCD18}
by modifying the algorithm in~\cite{CoudertMadre93} for computing prime implicant 
covers; see also~\cite{PrimeOverview93,minato1993fast}. The modified algorithm
optimizes the original one by integrating the instance into the prime implicant enumeration
process, but we are unaware of a complexity bound for the original algorithm or its modification. 
Moreover, since the algorithm is based on prime implicant covers, it is incomplete. 
Consider classifier \(\Delta = (X \wedge Z) \vee (Y \wedge \neg Z)\), which has three prime implicants:
\((X \wedge Z)\), \((Y \wedge \neg Z)\) and \((X \wedge Y)\). The last prime implicant is
redundant and may not be generated when computing a cover.
Instance \(\alpha = X, Y, Z\) leads to a positive decision and two sufficient reasons:
\((X \wedge Z)\) and \((X \wedge Y)\). An algorithm based on
covers may miss the sufficient reason \((X \wedge Y)\) and is therefore incomplete.
This can be problematic for queries that rely on examining all sufficient reasons, such
as decision and classifier bias (Definitions~\ref{def:bias-d} and~\ref{def:bias-c}).

We next propose a new approach based on computing the complete reason \(\R\) for a 
decision (Definition~\ref{def:the}), which characterizes all sufficient reasons, and then use
it to compute multiple queries. For example, we can enumerate all sufficient reasons using 
the reason \(\R\) (Theorem~\ref{theo:the-pi}). We can also use it to compute the necessary 
reason for a decision (Proposition~\ref{prop:necessary2}) and to detect decision bias (Theorem~\ref{theo:bias-d-exist}). 
Even classifier bias can sometimes be inferred directly using the reason \(\R\) (Theorem~\ref{theo:bias-c-cond}) among other queries.

Assuming the classifier is represented using a suitable tractable circuit (e.g., OBDD), our approach
will compute the complete reason for a decision in linear time regardless of how many sufficient
reasons it may have (could be exponential). Moreover, it will ensure that the computed
complete reason is represented by a tractable circuit, allowing us to answer many queries in polytime.

\subsection{Computing Complete Reasons}

Our approach is based on Decision-DNNF circuits, obtained using compilers
such as \ctd\footnote{\url{http://reasoning.cs.ucla.edu/c2d/}}~\cite{Darwiche04},
\mctd\footnote{\url{http://reasoning.cs.ucla.edu/minic2d/}}~\cite{OztokD14,OztokD18} and 
\df\footnote{\url{http://www.cril.univ-artois.fr/kc/d4.html}}~\cite{LagniezM17}.

\begin{definition}[\bf Decision-NNF Circuit]
\label{def:d-nnf}
A DNNF circuit has literals or constants as inputs and two type of gates: and-gates and or-gates,
where the subcircuits feeding into each and-gate share no variables. It is called a Decision-DNNF 
circuit if every or-gate has exactly two inputs of the form:  \(X \wedge \mu\) and \(\neg X \wedge \nu\), where \(X\) is a variable.
\end{definition}

DNNF circuits were introduced in~\cite{DarwicheJACM01}. Decision-DNNF circuits were identified
in~\cite{HuangD05,HuangD07} and include Ordered Binary Decision Diagrams (OBDDs)~\cite{Bryant86,HuangD07}.
Figure~\ref{fig:obdd2nnf} depicts an OBDD and its corresponding Decision-DNNF circuit. The circuit is obtained
by mapping each OBDD node with variable \(X\), high child \(\mu\) and low child \(\nu\) into the circuit fragment
\((X\wedge \mu) \vee (\neg X \wedge \nu)\) (two and-gates and one or-gate). 
For more on DNNF circuits and OBDD, see~\cite{DarwicheJAIR02,OztokD14}.

We compute the reason behind decision \(\Delta(\alpha)\) by applying two operations to
a Decision-DNNF circuit \(\Delta_\alpha\): {\em consenus} then {\em filtering.} 

\begin{definition}[\bf Consensus Circuit]
\label{def:c-circuit}
The consensus circuit of Decision-DNNF circuit \(\Gamma\) is denoted \(\cons(\Gamma)\)
and obtained by adding input \(\mu \wedge \nu\) to every or-gate with inputs \(X \wedge \mu\) and \(\neg X \wedge \nu\).
\end{definition}
Figure~\ref{fig:obdd2nnf} depicts a Decision-DNNF circuit and its consensus circuit (third from left). The
consensus operation adds four and-gates denoted with double circles.

\begin{proposition}\label{prop:consensus}
A Decision-DNNF circuit \(\Gamma\) has the same satisfying assignments as its consensus circuit \(\cons(\Gamma)\).
\end{proposition}
\begin{proof}
\((X\wedge\mu)\vee(\neg X\wedge\nu) \equiv (X\wedge\mu)\vee(\neg X\wedge\nu)\vee(\mu\wedge\nu)\).
\end{proof}
A consensus circuit can be obtained in time linear in the size of Decision-DNNF circuit,
but is not a DNNF circuit. We next discuss the filtering of a consensus circuit, which 
leads to a tractable circuit. 

\begin{definition}[\bf Filtered Circuit]
\label{def:f-circuit}
The filtering of consensus circuit \(\Gamma\) by instance \(\alpha\), where \(\Gamma(\alpha)=1\), 
is denoted \(\filter(\Gamma,\alpha)\) and obtained by replacing every literal \(l \not \in \alpha\) by constant \(0\).
\end{definition}
Filtering is only defined on consensus circuits and requires an instance that satisfies the circuit. 
Figure~\ref{fig:obdd2nnf} depicts an example. 
The filtered circuit is on the far right of the figure, where grayed out nodes and edges 
can be dropped due to replacing literals by constant \(0\).

Filtering is also a linear time operation.
Consensus preserves models, but filtering drops some of them.
We will characterize the models preserved by filtering after presenting two required results.

Let \(\Gamma\) be a circuit that results from filtering by instance \(\alpha\).
The circuit is monotone in the following sense.
If instance \(\gamma\) agrees with instance \(\alpha\) no less than instance \(\beta\) does, 
then \(\beta \models \Gamma\) implies \(\gamma \models \Gamma\).
For example, if \(\alpha = X,Y,Z\), \(\beta = \neg X, Y, \neg Z\) and \(\gamma = \neg X, Y, Z\).

\begin{thm}\label{theo:monotone}
If circuit \(\Gamma\) results from filtering by instance \(\alpha\) then every literal \(l\) in \(\Gamma\)
appears in \(\alpha\), and \(\Gamma(\gamma) \geq \Gamma(\beta)\) if \(\gamma \cap \alpha \supseteq \beta \cap \alpha\).
\end{thm}
\begin{proof}
Filtering removes every literal not in instance \(\alpha\).
Hence, every literal in the filtered circuit \(\Gamma\) is in \(\alpha\), which implies the next result.

Suppose that \(\gamma \cap \alpha \supseteq \beta \cap \alpha\) and \(\Gamma(\beta)=1\).
When evaluating circuit \(\Gamma\) at \(\gamma\) compared to \(\beta\), the only literals that change values are 
\(l_1 \in \gamma \setminus \beta\) and \(l_2 \in \beta \setminus \gamma\). Literals \(l_1\) change
values from \(0\) to \(1\) and literals \(l_2\) change values from \(1\) to \(0\). Changes to the values of \(l_1\) 
cannot decrease the output of circuit \(\Gamma\) since it is an NNF circuit. Literals \(l_2\) are not in \(\alpha\)
since \(\gamma \cap \alpha \supseteq \beta \cap \alpha\)
so do not appear in circuit \(\Gamma\) and changes to their values do not matter. Hence,
\(\Gamma(\gamma) = 1\).
\end{proof}

\shrink{
If an instance \(\beta\) satisfies a circuit that has been filtered by instance~\(\alpha\), then it is 
because of the property that \(\beta\) shares with \(\alpha\).
\begin{thm}\label{theo:property}
If circuit \(\Gamma\) results from filtering by instance \(\alpha\) and if \(\beta\) is an instance
such that \(\beta \models \Gamma\), then \(\alpha \cap \beta \models \Gamma\).
\end{thm}
\begin{proof}
Suppose \(\beta \models \Gamma\) and let \(\t = \alpha \cap \beta\).
For every instance \(\gamma\) such that \(\gamma \models \t\), we have \(\gamma \cap \alpha \supseteq \beta \cap \alpha\).
Hence, \(\Gamma(\gamma) \geq \Gamma(\beta)\) by Theorem~\ref{theo:monotone} and \(\gamma \models \Gamma\).
Hence \(\t\) is an implicant of circuit \(\Gamma\).
\end{proof}
}

We also need the following result which identifies circuit models that are preserved by filtering due to having applied consensus. 
\begin{proposition}\label{prop:drop}
Consider a Decision-DNNF circuit \(\Delta\) and instance \(\alpha\) such that \(\Delta(\alpha)=1\).
If \(\t\) is an implicant of \(\Delta\) and \(\alpha \models \t\) then \(\t\) is also an implicant of \(\filter(\cons(\Delta),\alpha)\).
\end{proposition}
\begin{proof}
\def\imp{{\cal I}}
Let \(\Gamma = \filter(\cons(\Delta),\alpha)\), \(\imp(\Delta) = \{\t: \t \models \Delta\}\) and  
\(\imp(\Delta,\alpha) = \{\t: \t \models \Delta \mbox{ and } \alpha \models \t\}\).
We need to show that \(\imp(\Delta,\alpha) \subseteq \imp(\Gamma)\). 
That is, \(\Gamma\) preserves the implicants \(\t\) of \(\Delta\) that are satisfied by \(\alpha\). 
The proof is by induction on the structure of \(\Delta\).

(Base Case)~If \(\Delta\) is a literal \(l\) or a constant, then \(\Delta = \Gamma\) since 
consensus is not applicable and filtering will not replace literal \(l\) by constant \(0\) (\(l \in \alpha\)
since \(\Delta(\alpha)=1\)). Hence, \(\imp(\Delta,\alpha) \subseteq \imp(\Gamma)\). 

(Inductive Step)~If \(\Delta = \Delta_1 \wedge \Delta_2\) then \(\Gamma= \Gamma_1\wedge \Gamma_2\)
where \(\Gamma_1 = \filter(\cons(\Delta_1),\alpha)\) and \(\Gamma_2= \filter(\cons(\Delta_2),\alpha)\).
Since \(\Delta_1\) and \(\Delta_2\) do not share variables (decomposability), 
\(\imp(\Delta) = \imp(\Delta_1) \times \imp(\Delta_2)\) (Cartesian product). 
Similarly, \(\imp(\Gamma) = \imp(\Gamma_1) \times \imp(\Gamma_2)\).
By the induction hypothesis,
\(\imp(\Delta_1,\alpha) \subseteq \imp(\Gamma_1)\) and \(\imp(\Delta_2,\alpha) \subseteq \imp(\Gamma_2)\).
Hence,  
\[ \imp(\Delta,\alpha) = \imp(\Delta_1,\alpha) \times \imp(\Delta_2,\alpha) \subseteq \imp(\Gamma_1) \times \imp(\Gamma_2) = \imp(\Gamma).\]

(Inductive Step)~If \(\Delta = (l \wedge\Delta_1) \vee (\neg l \wedge \Delta_2)\) 
and literal \(l \in \alpha\) then \(\Gamma = (l\wedge \Gamma_1) \vee (\Gamma_1 \wedge \Gamma_2)\)
where \(\Gamma_1 = \filter(\cons(\Delta_1),\alpha)\) and \(\Gamma_2 = \filter(\cons(\Delta_2),\alpha)\).
Due to decomposability, \(l\) and \(\neg l\) do not appear in \(\Delta_1\) or \(\Delta_2\).
Hence, \(\imp(\Delta) = \imp_1 \cup \imp_2 \cup \imp_c\) where
\begin{eqnarray*}
\imp_1 & = & \{l,\tau: \tau \in \imp(\Delta_1)\} \\
\imp_2 & = & \{\neg l,\tau: \tau \in \imp(\Delta_2)\} \\
\imp_c & = & \imp(\Delta_1 \wedge \Delta_2).
\end{eqnarray*}
Since \(\imp_2 \cap \imp(\Delta,\alpha) = \emptyset\) we have
\[
\imp(\Delta,\alpha) = \{l,\tau: \tau \in \imp(\Delta_1,\alpha)\} \cup  \imp(\Delta_1 \wedge \Delta_2,\alpha).
\]
Moreover,
\(
\imp(\Gamma) = \{l,\t: \t \in \imp(\Gamma_1)\} \cup \imp(\Gamma_1 \wedge \Gamma_2)
\).
By the induction hypothesis, \(\imp(\Delta_1,\alpha) \subseteq \imp(\Gamma_1)\) and \(\imp(\Delta_2,\alpha) \subseteq \imp(\Gamma_2)\),
which gives \(\{l,\tau: \tau \in \imp(\Delta_1,\alpha)\}  \subseteq \{l,\t: \t \in \imp(\Gamma_1)\}\) 
and \(\imp(\Delta_1 \wedge \Delta_2,\alpha) \subseteq \imp(\Gamma_1 \wedge \Gamma_2)\).
Hence, \(\imp(\Delta,\alpha) \subseteq \imp(\Gamma)\).
\end{proof}

The following fundamental result reveals the role of filtering a consensus circuit. It also reveals
our linear-time procedure for computing the complete reason behind a decision as a (tractable)
circuit that compactly characterizes all sufficient reasons.

\begin{thm}\label{theo:filter}
Consider a Decision-DNNF circuit \(\Delta\) and instance \(\alpha\) such that \(\Delta(\alpha)=1\).
Term \(\t\) is a prime implicant of \(\Delta\) and \(\alpha \models \t\) (that is, \(\t\) is a sufficient 
reason for decision \(\Delta(\alpha)\)) iff \(\t\) is a prime implicant of \(\filter(\cons(\Delta),\alpha)\).
\end{thm}
\begin{proof}
Let \(\Gamma = \filter(\cons(\Delta),\alpha)\) and observe that 
\(\Gamma \models \Delta\) since \(\cons(\Delta) \equiv \Delta\) and \(\Gamma\) is the result 
of replacing some inputs of NNF circuit \(\cons(\Delta)\) with constant \(0\).

Suppose \(\t\) is a prime implicant of circuit \(\Delta\) and \(\alpha \models \t\).
Then \(\t\) is an implicant of circuit \(\Gamma\) by Proposition~\ref{prop:drop}, \(\t \models \Gamma\).
If \(\t\) is not a prime implicant of \(\Gamma\), we must have some term \(\a \subset \t\) such that
\(\a \models \Gamma\). Therefore \(\a \models \Delta\) since \(\Gamma \models \Delta\),
which means that \(\t\) is not a prime implicant of \(\Delta\), a contradiction. Hence,
\(\t\) is a prime implicant of \(\Gamma\). 

Suppose \(\t\) is a prime implicant of circuit \(\Gamma\). Then \(\t\) is an implicant of \(\Delta\)
since \(\Gamma \models \Delta\). We next show that \(\t\) is a prime implicant of \(\Delta\) 
and \(\alpha \models \t\). Let \(\beta\) be an instance
such that \(\beta \models \t\) and \(\beta\) disagrees with \(\alpha\) on all variables outside \(\t\). 
Then \(\Gamma(\beta)=1\) and \(\alpha \cap \beta \subseteq \t\).
Every instance \(\gamma\) such that \(\gamma \models \alpha \cap \beta\) must satisfy 
\(\Gamma(\gamma)=1\) since \(\alpha \cap \gamma \supseteq \alpha \cap \beta\), leading to 
\(\Gamma(\gamma) \geq \Gamma(\beta)\) by Theorem~\ref{theo:monotone}.
Hence, \(\alpha \cap \beta\) is an implicant of \(\Gamma\). Since \(\t\) is a prime 
implicant of \(\Gamma\), we must have \(\alpha \cap \beta = \t\)  and hence \(\alpha \models \t\).
Suppose now \(\t\) is not a prime implicant of \(\Delta\). Some term \(\a \subset \t\) is then
a prime implicant of \(\Delta\) and \(\alpha \models \a\). By the first part of this theorem, 
\(\a\) is a prime implicant of \(\Gamma\), a contradiction. Therefore, \(\t\) is a prime 
implicant of \(\Delta\).
\end{proof}

\begin{definition}[\bf Reason Circuit]
\label{def:reason-c}
For classifier \(\Delta\), instance \(\alpha\) and a Decision-DNNF circuit \(\Gamma\) for \(\Delta_\alpha\),
circuit \(\filter(\cons(\Gamma),\alpha)\) is called a ``reason circuit'' and denoted \(\reason(\Delta,\alpha)\).
\end{definition}
The circuit \(\reason(\Delta,\alpha)\) depends on the specific Decision-DNNF circuit \(\Gamma\) used
to represent \(\Delta_\alpha\) but will always have the same models.

\subsection{Tractability of Reason Circuits}

We next show that reason circuits are tractable. Since we represent the complete
reason for a decision as a reason circuit, many queries relating to the decision can then be answered efficiently.

\begin{definition}[\bf Monotone]
\label{def:montone}
An NNF circuit is monotone if every variable appears only positively or only negatively in the circuit.
\end{definition}
Reason circuits are filtered circuits and hence monotone as shown by Theorem~\ref{theo:monotone}.
The following theorem mirrors what is known on monotone propositional formula, 
but we include it for completeness.

\begin{thm}\label{theo:tractable}
The satisfiability of a monotone NNF circuit can be decided in linear time.
A monotone NNF circuit can be negated and conditioned in linear time to yield a monotone NNF circuit.
\end{thm}
\begin{proof}
The satisfiability of a monotone NNF circuit can be decided using the following procedure.
Constant \(0\) is not satisfiable. Constant \(1\) and literals are satisfiable. An or-gate is satisfiable
iff any of its inputs is satisfiable. An and-gate is satisfiable iff all its inputs are satisfiable. All previous
statements are always correct except the last one which depends on monotonicity. Consider a
conjunction \(\mu \wedge \nu\) and suppose every variable shared between the conjuncts
appears either positively or negatively in both. Any model of \(\mu\) can be combined with any
model of \(\nu\) to form a model for \(\mu \wedge \nu\). Hence, the conjunction is satisfiable
iff each of the conjuncts is satisfiable. 
Conditioning replaces literals by constants so it preserves monotonicity.
To negate a monotone circuit, replace and-gates by or-gates, or-gates by and-gates
and literals by their negations. Monotonicity is preserved.
\end{proof}
Given Theorem~\ref{theo:tractable}, the validity of a monotone NNF circuit can be decided 
in linear time (we check whether the negated circuit is unsatisfiable).\footnote{Validity can be checked more directly as follows.
Constant \(1\) is valid. Constant \(0\) and literals are not valid. An and-gate is valid iff all
its inputs are valid. An or-gate is valid iff any of its inputs is valid. The previous statements
are always correct except the last one which requires monotonicity.} 
We can also conjoin the circuit with a literal in linear time to yield a monotone circuit since 
\(\Delta \wedge l = (\Delta \vert l) \wedge l\).

Variables can be existentially quantified from a monotone circuit in linear time, with the resulting
circuit remaining monotone. This is critical for efficiently detecting decision bias as shown 
by Theorem~\ref{theo:bias-d-exist}.

\begin{thm}\label{theo:exists}
Replacing every literal of variable \(X\) with constant \(1\) in monotone NNF circuit \(\Gamma\)  
yields a circuit equivalent to \(\exists X \Gamma\). 
\end{thm}
\begin{proof}
If variable \(X\) appears only positively in circuit \(\Gamma\) then \(\Gamma \vert \neg X \models \Gamma \vert X\)
and \(\exists X \; \Gamma = (\Gamma \vert X) \vee (\Gamma \vert \neg X) = \Gamma \vert X\).
If variable \(X\) appears only negatively in \(\Gamma\) then \(\Gamma \vert X \models \Gamma \vert \neg X\) and
\(\exists X \; \Gamma = (\Gamma \vert X) \vee (\Gamma \vert \neg X) = \Gamma \vert \neg X\). 
Variable \(X\) can therefore be existentially quantified by replacing its literals with constant \(1\).
\end{proof}

\subsection{Computing Queries}
\label{sec:queries}

\def\PI{{\sf PI}}
\def\cache{{\sf cache}}
\def\cp{{\sf cartesian\_product}}
\def\rs{{\sf remove\_subsumed}}

\begin{algorithm}[tb]
\small
\caption{\PI(\(\Delta,\alpha\)) \label{alg:PI}}

\noindent {\bf input:} Decision-DNNF circuit \(\Delta\), instance \(\alpha\) (assumes \(\Delta(\alpha)=1\)).

\noindent {\bf output:} Prime implicants of circuit \(\filter(\cons(\Delta),\alpha)\).
\vspace{-.5mm}
\Amain{
\begin{algorithmic}[1]

\IF{\(\cache(\Delta)\) is set}
  \RETURN \(\cache(\Delta)\)
\ELSIF{\(\Delta\) is constant \(0\)}
  \STATE \(r = \{\}\)
\ELSIF{\(\Delta\) is constant \(1\)}
  \STATE \(r = \{\{\}\}\)
 \ELSIF{\(\Delta = \Delta_1 \wedge \Delta_2\)}
  \STATE \(r = \cp(\PI(\Delta_1,\alpha),\PI(\Delta_2,\alpha))\)
 \ELSIF{\(\Delta = (X\wedge \Delta_1) \vee (\neg X \wedge \Delta_2)\)}
  \STATE \((\ell, \Gamma) = (X,\Delta_1) \mbox{ if literal \(X\) in \(\alpha\) else } (\neg X,\Delta_2)\) \label{ln:cf}
  \STATE \(p =  \cp(\PI(\Delta_1,\alpha),\PI(\Delta_2,\alpha))\) \label{ln:cp}
  \STATE \(q = \{\{\ell\} \cup \t \mbox{ for } \t \in \PI(\Gamma,\alpha)\}\)
  \STATE \(r = p \cup q\) \label{ln:s}
\ENDIF
\STATE \(r = \rs(r)\)
\STATE \(\cache(\Delta) = r\)
\RETURN \(r\)
\end{algorithmic}
}
\end{algorithm}

We can now discuss algorithms. To compute the sufficient reasons for a decision \(\Delta(\alpha)\):
get a Decision-DNNF circuit for \(\Delta_\alpha\), transform it into a consensus circuit, filter it 
by instance \(\alpha\) and finally compute the prime implicants of filtered circuit. Algorithm~\ref{alg:PI}
does this in place, that is without explicitly constructing the consensus or filtered circuit. It assumes
a positive decision (otherwise we pass \(\neg \Delta\)).

Algorithm~\ref{alg:PI} uses subroutine \(\cp\) which conjoins two DNFs by computing
the Cartesian product of their terms. It also uses \(\rs\) to remove subsumed terms from a DNF. 

\begin{thm}\label{theo:pi}
Consider a Decision-DNNF \(\Delta\) and instance \(\alpha\). If \(\Delta(\alpha)=1\) then a call
\(\PI(\Delta,\alpha)\) to Algorithm~\ref{alg:PI} returns the prime implicants of circuit \(\filter(\cons(\Delta),\alpha)\).
\end{thm}
\begin{proof}
Consensus and filtering are applied implicitly on Lines~\ref{ln:cf}-\ref{ln:cp}. Filtered circuit are
monotone. We compute the prime implicants of a monotone circuit by converting 
it into DNF and removing subsumed terms~\cite[Chapter 3]{BooleanFunctions}.
This is precisely what Algorithm~\ref{alg:PI} does.
\end{proof}

Consider a decision \(\Delta(\alpha)\) and its complete reason \(\R = \reason(\Delta,\alpha)\) as
a monotone NNF circuit obtained by consensus then filtering. 
Let \(n\) be the size of circuit \(\R\) and \(m\) be the number of features.
We next show how to compute various queries using circuit~\(\R\).

\vspace{2mm}
\noindent {\bf Sufficient Reasons.}
By Theorems~\ref{theo:the-pi} and~\ref{theo:pi}, the call \(\PI(\Delta_\alpha,\alpha)\) 
to Algorithm~\ref{alg:PI} will return all sufficient reasons for decision \(\Delta(\alpha)\), 
assuming \(\Delta_\alpha\) is a Decision-DNNF circuit. The number of sufficient
reasons can be exponential, but we can actually answer many questions about
them without enumerating them directly as shown below.

\vspace{2mm}
\noindent {\bf Necessary Property.}
By Proposition~\ref{prop:necessary2}, characteristic (literal) \(l\) is necessary for decision \(\Delta(\alpha)\) 
iff \(\R \models l\). This is equivalent to \(\R \vert \neg l\) being unsatisfiable, which can be 
decided in \(O(n)\) time given Theorem~\ref{theo:tractable}. The necessary
property (all necessary characteristics) can then be computed in \(O(n \cdot m)\) time.

\vspace{2mm}
\noindent {\bf Necessary Reason.}
To compute the necessary reason (if any) we compute the necessary property and
check whether it satisfies the complete reason. This can be done in \(O(n \cdot m)\) time.

\vspace{2mm}
\noindent{\bf Because Statements.}
To decide whether decision \(\Delta(\alpha)\) was made ``because \(\t\)'' we check
whether property \(\t\) is the complete reason for the decision (Definition~\ref{def:because}):
\(\t \models \R\) and \(\R \models \t\). We have \(\t \models \R\) iff \((\neg \R) \vert \t\) is unsatisfiable.
Moreover, \(\R \models \t\) iff \(\R \vert \neg l\) is unsatisfiable for every literal \(l\) in \(\t\). 
All of this can be done in \(O(n \cdot \vert \t \vert)\) time.

\vspace{2mm}
\noindent{\bf Even if, Because Statements.}
To decide whether decision \(\Delta(\alpha)\) would stick ``even if \(\na\) because \(\t\)'' we replace
property \(\a\) with \(\na\) in instance \(\alpha\) to yield instance \(\beta\) (Definition~\ref{def:eib}).
We then compute the complete reason for decision \(\Delta(\beta)\) and
check whether it is equivalent to \(\t\). All of this can be done \(O(n \cdot \vert \t \vert)\) time.

\vspace{2mm}
\noindent{\bf Decision Bias.}
To decide whether decision \(\Delta(\alpha)\) is biased we existentially quantify all unprotected
features from circuit \(\R\) and then check the validity of the result (Theorem~\ref{theo:bias-d-exist}).
All of this can be done in \(O(n)\) time given Theorems~\ref{theo:tractable} and~\ref{theo:exists}.

\shrink{
\vspace{2mm}
\noindent{\bf Classifier Bias.}
A classifier is biased iff the reason behind one of its decisions depends on at least one
protected feature (Theorem~\ref{theo:bias-c-cond}). To check whether circuit \(\R\)
depends on protected feature \(X\), that is whether \(\R \vert X \neq \R \vert \neg X\),
we first existentially quantify all other features from \(\R\) to yield monotone NNF
circuit \(\R^\star\). We then check whether \(\R^\star \vert X \neq \R^\star \vert \neg X\).
All of this can be done in \(O(n)\) time.\comment(this is wrong)
}

\section{Another Admissions Classifier}
\label{sec:study}

\begin{figure}[t]
\begin{minipage}[b]{0.4\linewidth}
\centering
\includegraphics[height=0.2\textheight]{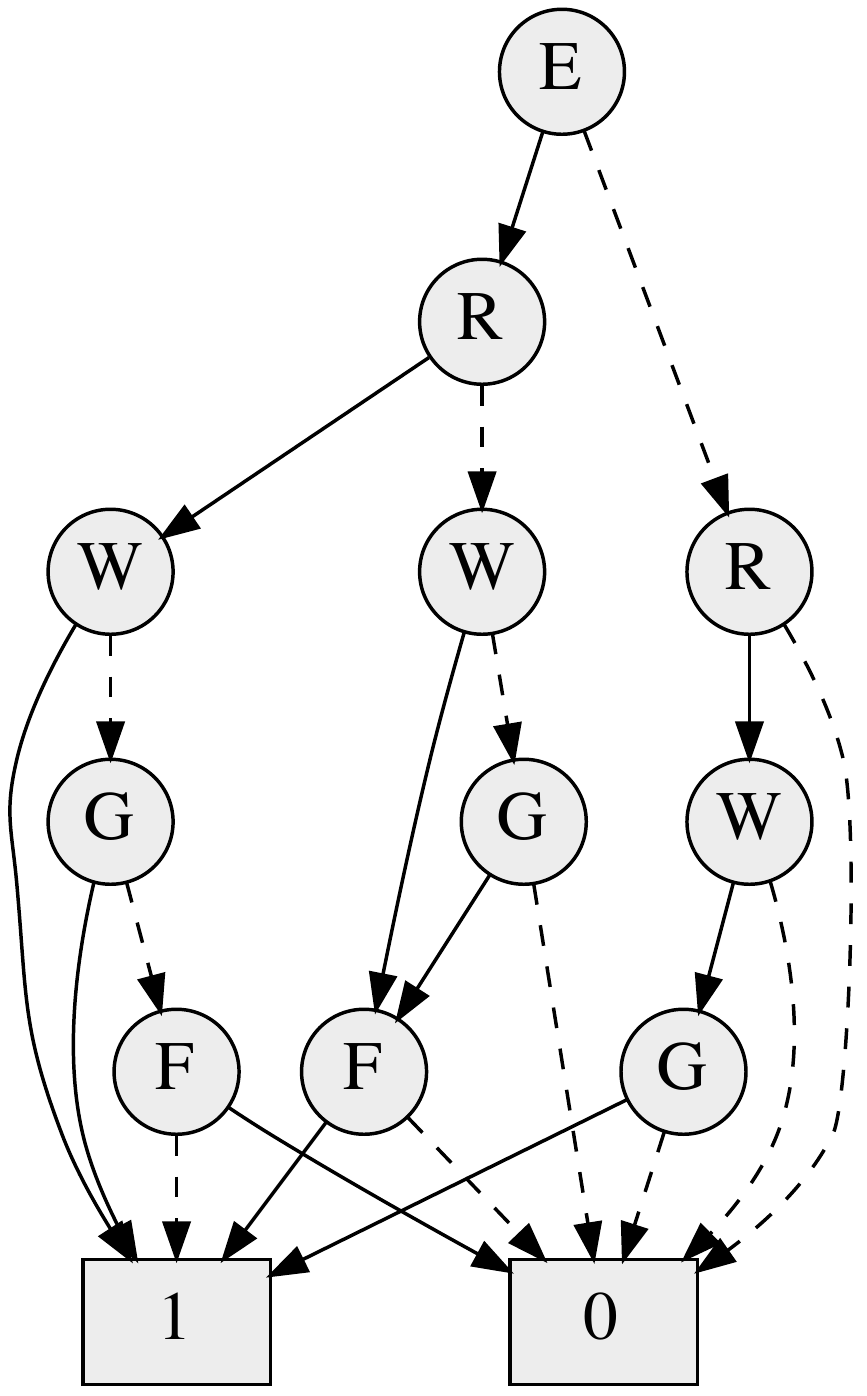}
\caption{Admission classifier.\label{fig:nonmon}}
\end{minipage}
\begin{minipage}[b]{0.4\linewidth}
\begin{footnotesize}
\begin{tabular}{|c||c|c|c|} \hline
Applicant & \rotatebox[origin=c]{270}{Scott} & \rotatebox[origin=c]{270}{ Robin } & \rotatebox[origin=c]{270}{April } \\ \hline
{\bf E}ntrance Exam &  \cmark & \cmark & \cmark \\
{\bf F}irst Time Applicant & \xmark & \cmark & \cmark \\
{\bf G}ood GPA &  \cmark & \cmark & \cmark \\
{\bf W}ork Experience & \cmark & \cmark & \cmark \\ 
{\bf R}ich Hometown & \cmark & \cmark & \xmark \\ \hline
Decision & 1 & 1 & 1 \\ \hline
\end{tabular}
\end{footnotesize}
\caption{Applicants and their characteristics. \label{tab:app}}
\end{minipage}
\vspace{5mm}

\centering
\includegraphics[height=0.22\textheight]{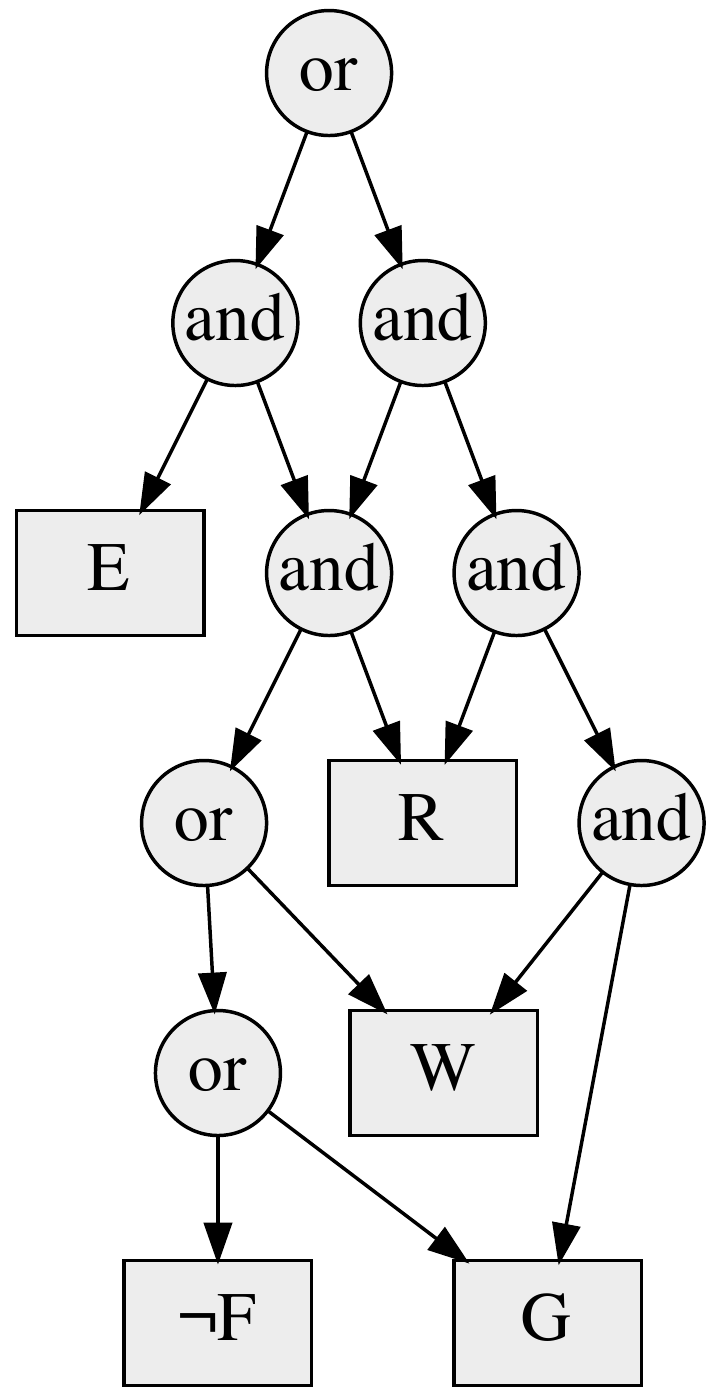}
\hspace{2mm}
\includegraphics[height=0.22\textheight]{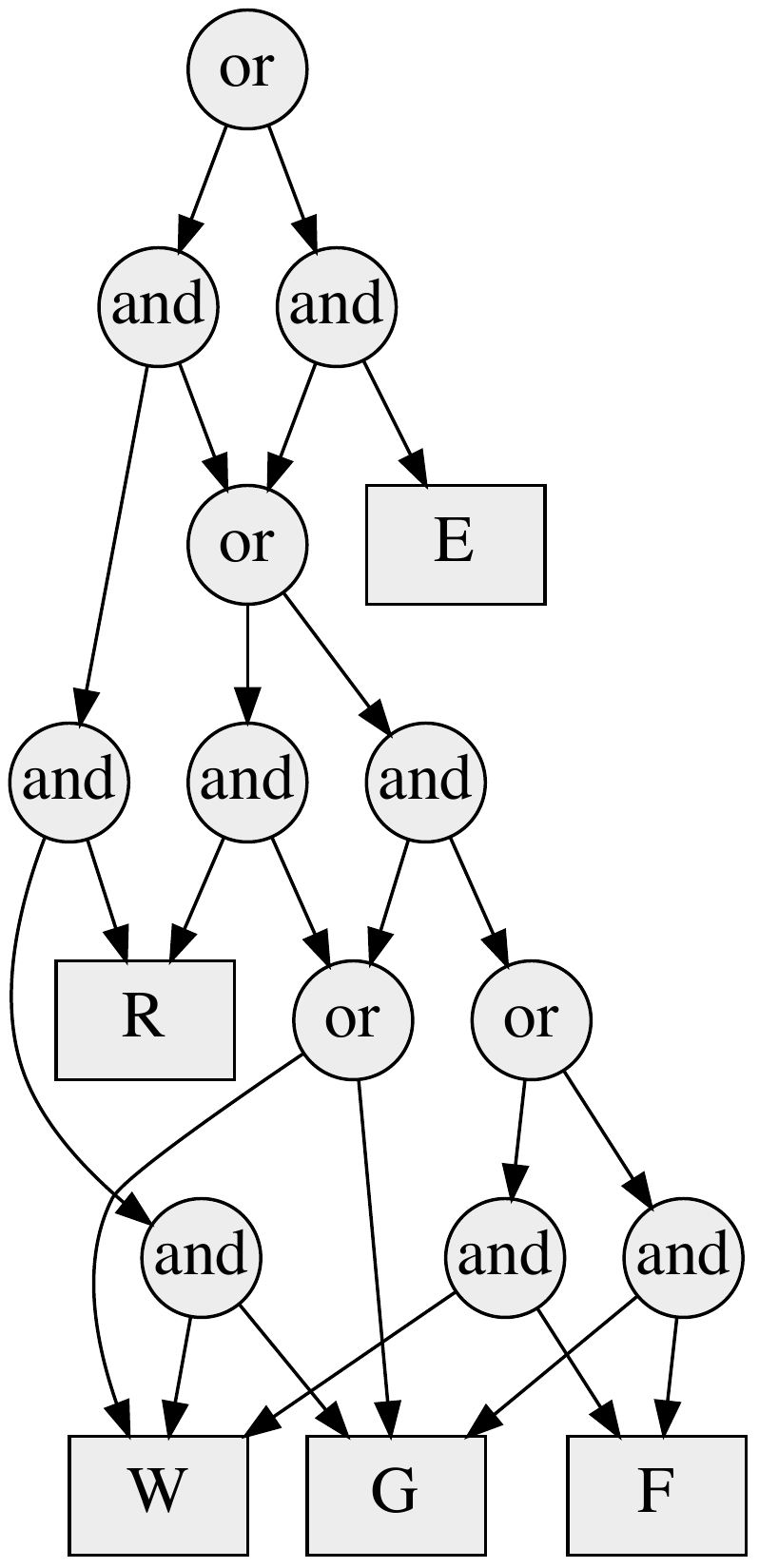}
\hspace{2mm}
\includegraphics[height=0.22\textheight]{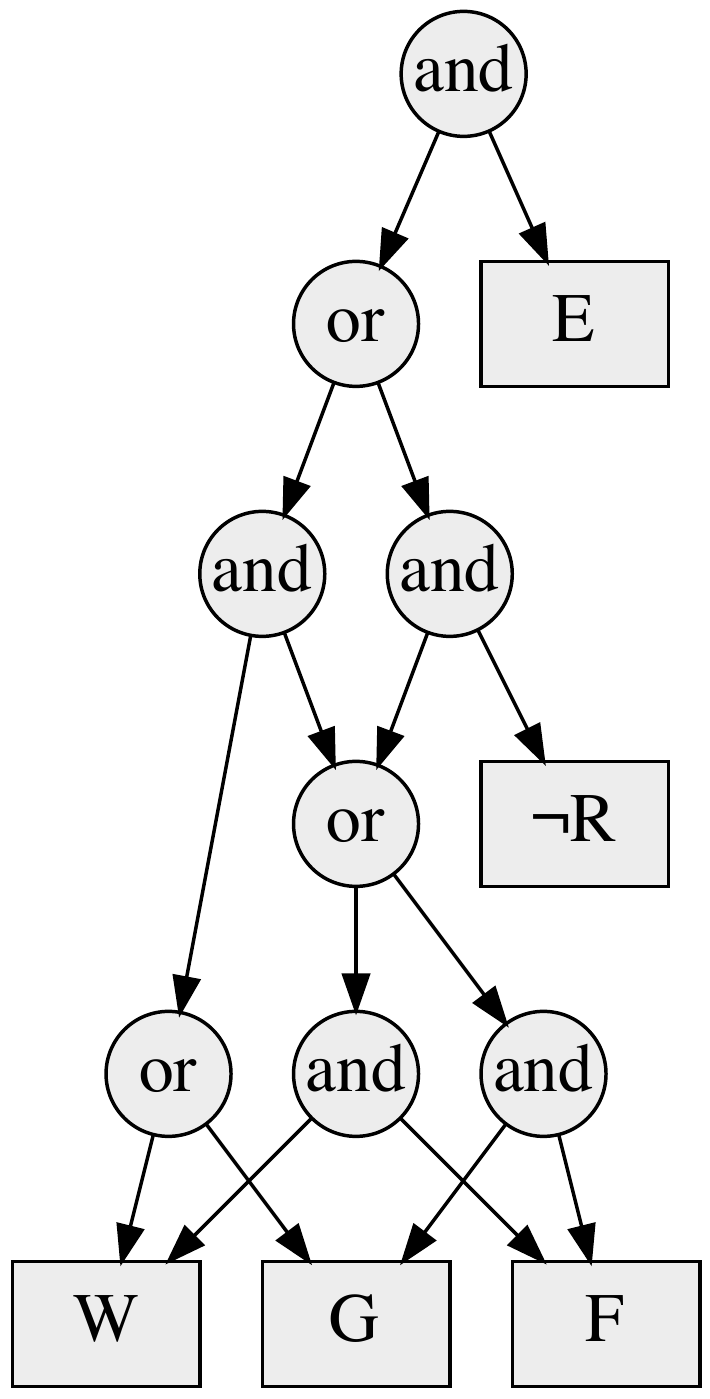}
\caption{From left to right: Reason circuit for the decision on applicants 
Scott, Robin and April (Figure~\ref{tab:app}).
\label{fig:reasons}}
\end{figure}

We now consider a more refined admission classifier to illustrate the notions and concepts we
introduced more comprehensively.

This classifier highly values passing the entrance exam and being a first time applicant. 
However, it also gives significant leeway to students from a rich hometown. In fact, being 
from a rich hometown unlocks the only path to acceptance for those who failed the entrance exam. 
The classifier is depicted as an OBDD in Figure~\ref{fig:nonmon}. 
It corresponds to the following Boolean formula, which is not monotone 
(the previous classifiers we considered were all monotone):
\[
\Delta = [E \wedge [(F \wedge (G \vee W)) \vee (\neg F \wedge R)]] \vee [G \wedge R \wedge W].
\]
The classifier has the following prime implicants, some are not essential (all
prime implicants of a monotone formula are essential):
\[
(E,F,W)(E,F,G)(G,R,W)(E,\neg F,R)(E,R,W)(E,G,R).
\]
We will consider applicants Scott, Robin and April in Figure~\ref{tab:app}, where feature \(R\) is protected (whether 
the applicant comes from a rich hometown). The complete reasons for the
decisions on these applicants are shown in Figure~\ref{fig:reasons}. These are reason circuits
produced as suggested by Definition~\ref{def:reason-c}, except that we simplified the circuits
by propagating and removing constant values (a reason circuit is satisfiable as it must be satisfied by 
the instance underlying the decision).

The decision on applicant Scott is {\em biased.} To check this, we can existentially quantify unprotected features
\(E, F, G, W\) from the reason circuit in Figure~\ref{fig:reasons} and then check its 
validity (Theorem~\ref{theo:bias-d-exist}).
Existential quantification is done by replacing the literals \(E, \neg F, G, W\) in the circuit
with constant \(1\). The resulting circuit is not valid. We can also confirm decision bias
by considering the sufficient reasons for this decision, which all contain the protected 
feature \(R\) (Theorem~\ref{theo:bias-d}):
\[
(E, G, R) \:
(E, R, W)\:
(E, R, \neg F)\:
(G, R, W)
\]
If we flip the protected characteristic \(R\) to \(\neg R\), the decision will flip with
the complete reason being \(\neg F, \neg R\) so
Scott would be denied admission {\em because} he is not a first time applicant
and does not come from a rich hometown (Definition~\ref{def:because}). 

The decision on Robin is {\em not biased.} If we existentially quantify unprotected
features \(E, F, G, W\) from the reason circuit (by replacing their literals
with constant \(1\)), the circuit becomes valid.
We can confirm this by examining the decision's sufficient reasons:
\[
(E, F, G)\:
(E, F,  W)\:
(E, G,  R)\:
(E, R,  W)\:
(G, R, W)
\]
Two of these sufficient reasons do not contain the protected feature so the decision
cannot be biased (Theorem~\ref{theo:bias-d}). The decision will be the same on
any applicant with the same characteristics as Robin except for the protected feature \(R\). However,
since some of the sufficient reasons contain a protected feature, the classifier
must be biased (Theorem~\ref{theo:bias-c}): It will make a biased decision on some
other applicant. This illustrates how classifier bias can be inferred from
the complete reason behind one of its {\em unbiased} decisions. This method is not complete
though: the classifier may still be biased even if no protected feature appears
in a sufficient reason for one of its decisions.

The decision on April is {\em not biased} even though the protected feature \(R\)
appears in the reason circuit (the circuit is valid if we existentially quantify
all features but \(R\)). Moreover, \(E,F\) are all the necessary characteristics
for this decision (i.e., the necessary property). Flipping either of these characteristics
will flip the decision. Recall that violating the necessary property may either flip the decision or change the
reason behind it (Theorem~\ref{theo:necessary}) but flipping only
one necessary characteristic is guaranteed to flip the decision
(Proposition~\ref{prop:necessary1}).

The decision on April would stick {\em even if} she were not to have work experience (\(\neg W\))
{\em because} she passed the entrance exam (\(E\)), has a good GPA (\(G\)) and is a first time applicant (\(F\)). 
April would be denied admission if she were to also violate one of these characteristics (Definition~\ref{def:eib} and Proposition~\ref{prop:necessary1}).

We close this section by an important remark. Even though most of the notions
we defined are based on prime implicants, our proposed theory does not necessarily require the
computation of prime implicants which can be prohibitive. 
Reason circuits characterize all relevant prime implicants and can be obtained in linear time from
Decision-DNNF circuits. Reason circuits are also monotone, allowing one to answer many 
queries about the embedded prime implicants in polytime.
This is a major contribution of this work.

\section{Conclusion}
\label{sec:conclusion}

We introduced a theory for reasoning about the decisions of Boolean classifiers,
which is based on the notions of sufficient, necessary and complete reasons.
We presented applications of the theory to explaining decisions, evaluating 
counterfactual statements about decisions and identifying decision bias and
classifier bias. We also presented polytime and linear-time algorithms for 
computing most of the introduced notions based on the new and tractable class of
reason circuits.

\ack We wish to thank Arthur Choi and Jason Shen for providing valuable feedback.
This work has been partially supported by NSF grant \#ISS-1910317, 
ONR grant \#N00014-18-1-2561, DARPA XAI grant \#N66001-17-2-4032
and a gift from JP Morgan. 
The views in this paper do not necessarily represent those of sponsors.

\bibliography{pi_bibliography}

\begin{thebibliography}{10}

\bibitem{Bryant86}
Randal~E. Bryant, `Graph-based algorithms for boolean function manipulation',
  {\em {IEEE} Trans. Computers}, {\bf 35}(8),  677--691, (1986).

\bibitem{ChanD03}
Hei Chan and Adnan Darwiche, `Reasoning about bayesian network classifiers', in
  {\em {UAI}}, pp. 107--115. Morgan Kaufmann, (2003).

\bibitem{ChoiShiShihDarwiche19}
Arthur Choi, Weijia Shi, Andy Shih, and Adnan Darwiche, `Compiling neural
  networks into tractable {B}oolean circuits', in {\em AAAI Spring Symposium on
  Verification of Neural Networks (VNN)}, (2019).

\bibitem{CoudertMadre93}
Olivier Coudert and Jean~Christophe Madre, `Fault tree analysis: $10^{20}$
  prime implicants and beyond', in {\em Proc. of the Annual Reliability and
  Maintainability Symposium}, (1993).

\bibitem{PrimeOverview93}
Olivier Coudert, Jean~Christophe Madre, Henri Fraisse, and Herve Touati,
  `Implicit prime cover computation: An overview', in {\em Proceedings of the
  4th SASIMI Workshop}, (1993).

\bibitem{BooleanFunctions}
Yves Crama and Peter~L. Hammer, {\em Boolean Functions - Theory, Algorithms,
  and Applications}, volume 142 of {\em Encyclopedia of mathematics and its
  applications}, Cambridge University Press, 2011.

\bibitem{DarwicheJACM01}
Adnan Darwiche, `Decomposable negation normal form', {\em J. {ACM}}, {\bf
  48}(4),  608--647, (2001).

\bibitem{Darwiche04}
Adnan Darwiche, `New advances in compiling {CNF} into decomposable negation
  normal form', in {\em {ECAI}}, pp. 328--332. {IOS} Press, (2004).

\bibitem{DarwicheJAIR02}
Adnan Darwiche and Pierre Marquis, `A knowledge compilation map', {\em J.
  Artif. Intell. Res.}, {\bf 17},  229--264, (2002).

\bibitem{diagnosisPI}
Johan de~Kleer, Alan~K. Mackworth, and Raymond Reiter, `Characterizing
  diagnoses and systems', {\em Artif. Intell.}, {\bf 56}(2-3),  197--222,
  (1992).

\bibitem{DBLP:journals/tocl/HerzigLM13}
Andreas Herzig, J{\'{e}}r{\^{o}}me Lang, and Pierre Marquis, `Propositional
  update operators based on formula/literal dependence', {\em {ACM} Trans.
  Comput. Log.}, {\bf 14}(3),  24:1--24:31, (2013).

\bibitem{HuangD05}
Jinbo Huang and Adnan Darwiche, `{DPLL} with a trace: From {SAT} to knowledge
  compilation', in {\em {IJCAI}}, pp. 156--162. Professional Book Center,
  (2005).

\bibitem{HuangD07}
Jinbo Huang and Adnan Darwiche, `The language of search', {\em J. Artif.
  Intell. Res.}, {\bf 29},  191--219, (2007).

\bibitem{IgnatievNM19}
Alexey Ignatiev, Nina Narodytska, and Joao Marques{-}Silva, `Abduction-based
  explanations for machine learning models', in {\em Thirty-Third {AAAI}
  Conference on Artificial Intelligence ({AAAI})}, pp. 1511--1519, (2019).

\bibitem{JoaoNIPS19}
Alexey Ignatiev, Nina Narodytska, and Joao Marques-Silva, `On relating
  explanations and adversarial examples', in {\em Advances in Neural
  Information Processing Systems 32},  15883--15893, Curran Associates, Inc.,
  (2019).

\bibitem{JoaoApp}
Alexey Ignatiev, Nina Narodytska, and Jo{\~{a}}o Marques{-}Silva, `On
  validating, repairing and refining heuristic {ML} explanations', {\em CoRR},
  {\bf abs/1907.02509}, (2019).

\bibitem{LagniezM17}
Jean{-}Marie Lagniez and Pierre Marquis, `An improved decision-dnnf compiler',
  in {\em {IJCAI}}, pp. 667--673. ijcai.org, (2017).

\bibitem{DBLP:journals/jair/LangLM03}
J{\'{e}}r{\^{o}}me Lang, Paolo Liberatore, and Pierre Marquis, `Propositional
  independence: Formula-variable independence and forgetting', {\em J. Artif.
  Intell. Res.}, {\bf 18},  391--443, (2003).

\bibitem{ReasonsMoral}
Felix Lindner and Katrin M{\"o}llney, `Extracting reasons for moral judgments
  under various ethical principles', in {\em KI 2019: Advances in Artificial
  Intelligence}, eds., Christoph Benzm{\"u}ller and Heiner Stuckenschmidt, pp.
  216--229, Cham, (2019). Springer International Publishing.

\bibitem{Marquis00}
Pierre Marquis, {\em Consequence finding algorithms}, volume~5 of {\em Handbook
  on Defeasible Reasoning and Uncertainty Management Systems}, chapter~2,
  41--145, Kluwer Academic Publisher, 2000.
\newblock Moral S. et Kohlas J. (eds.), Gabbay D. et Smets Ph. (series eds.).

\bibitem{mccluskey}
E.~J. {McCluskey}, `Minimization of boolean functions', {\em The Bell System
  Technical Journal}, {\bf 35}(6),  1417--1444, (Nov 1956).

\bibitem{minato1993fast}
Shin{-}ichi Minato, `Fast generation of prime-irredundant covers from binary
  decision diagrams', {\em IEICE Transactions on Fundamentals of Electronics,
  Communications and Computer Sciences}, {\bf 76}(6),  967--973, (1993).

\bibitem{OztokD14}
Umut Oztok and Adnan Darwiche, `On compiling {CNF} into decision-dnnf', in {\em
  {CP}}, volume 8656 of {\em Lecture Notes in Computer Science}, pp. 42--57.
  Springer, (2014).

\bibitem{OztokD18}
Umut Oztok and Adnan Darwiche, `An exhaustive {DPLL} algorithm for model
  counting', {\em J. Artif. Intell. Res.}, {\bf 62},  1--32, (2018).

\bibitem{quine1}
W.~V. Quine, `The problem of simplifying truth functions', {\em The American
  Mathematical Monthly}, {\bf 59}(8),  521--531, (1952).

\bibitem{quine2}
W.~V. Quine, `On cores and prime implicants of truth functions', {\em The
  American Mathematical Monthly}, {\bf 66}(9),  755--760, (1959).

\bibitem{LIME}
Marco~T{\'{u}}lio Ribeiro, Sameer Singh, and Carlos Guestrin, `"why should {I}
  trust you?": Explaining the predictions of any classifier', in {\em {KDD}},
  pp. 1135--1144. {ACM}, (2016).

\bibitem{ANCHOR}
Marco~T{\'{u}}lio Ribeiro, Sameer Singh, and Carlos Guestrin, `Anchors:
  High-precision model-agnostic explanations', in {\em {AAAI}}, pp. 1527--1535.
  {AAAI} Press, (2018).

\bibitem{ShihCD18b}
Andy Shih, Arthur Choi, and Adnan Darwiche, `Formal verification of bayesian
  network classifiers', in {\em {PGM}}, volume~72 of {\em Proceedings of
  Machine Learning Research}, pp. 427--438. {PMLR}, (2018).

\bibitem{ShihCD18}
Andy Shih, Arthur Choi, and Adnan Darwiche, `A symbolic approach to explaining
  bayesian network classifiers', in {\em {IJCAI}}, pp. 5103--5111. ijcai.org,
  (2018).

\bibitem{ShihCD19}
Andy Shih, Arthur Choi, and Adnan Darwiche, `Compiling bayesian network
  classifiers into decision graphs', in {\em {AAAI}}, pp. 7966--7974. {AAAI}
  Press, (2019).

\bibitem{ShihDC19b}
Andy Shih, Adnan Darwiche, and Arthur Choi, `Verifying binarized neural
  networks by angluin-style learning', in {\em {SAT}}, volume 11628 of {\em
  Lecture Notes in Computer Science}, pp. 354--370. Springer, (2019).

\end{thebibliography}

\end{document}